\documentclass[12pt]{article}
\usepackage[utf8]{inputenc}
\usepackage[margin=1in]{geometry}

\usepackage{times}
\usepackage{amsmath}
\usepackage{amsfonts}
\usepackage{amssymb}
\usepackage{amsthm}
\usepackage{thmtools}
\usepackage{thm-restate}
\usepackage{mathrsfs}
\usepackage{enumerate}
\usepackage{bbm}

\usepackage{natbib}
\usepackage{hyperref}
\hypersetup{
    colorlinks=true,
    linkcolor=blue,
    citecolor=blue,
    hypertexnames=false
}
\usepackage{cleveref}

\newtheorem{theorem}{Theorem}
\newtheorem{lemma}[theorem]{Lemma}
\newtheorem{proposition}[theorem]{Proposition}

\newtheorem{corollary}[theorem]{Corollary}
\theoremstyle{definition}
\newtheorem{definition}[theorem]{Definition}
\newtheorem{example}{Example}

\theoremstyle{remark}
\newtheorem{remark}[theorem]{Remark}

\DeclareMathOperator*{\argmin}{argmin}

\newcommand{\cX}{\mathcal{X}}
\newcommand{\rr}{\mathbb{R}}
\newcommand{\pp}{\mathbb{P}}
\newcommand{\cE}{\mathcal{E}}
\newcommand{\ee}{\mathbb{E}}
\newcommand{\I}{\mathbb{I}}
\newcommand{\F}{\mathcal{F}}
\newcommand{\norm}[1]{\left|\left| #1 \right|\right|}

\newcommand{\abs}[1]{\left| #1 \right|}

\DeclareMathOperator{\reg}{Reg}
\renewcommand{\epsilon}{\varepsilon}
\newcommand{\dkl}[2]{D_{KL}\left(#1 || #2\right)}
\newcommand{\dren}[2]{D_\lambda\left(#1 || #2\right)}
\newcommand{\chisq}[2]{\chi^2\left(#1 || #2\right)}
\renewcommand{\phi}{\varphi}
\newcommand{\divf}[2]{D_f\left(#1 || #2\right)}
\DeclareMathOperator{\tv}{TV}

\newcommand{\jstar}{j^{\ast}}

\newcommand{\scrF}{\mathscr{F}}
\newcommand{\nuM}{\nu_{M}}
\newcommand{\nutil}{\widetilde{\nu}}
\newcommand{\yhat}{\widehat{y}}
\newcommand{\vc}[1]{\mathrm{vc}\left( #1 \right)}
\newcommand{\rad}{\mathfrak{R}}
\newcommand{\fat}[1]{\mathrm{fat}\left( #1 \right)}
\newcommand{\radseq}{\rad^{\mathrm{seq}}}
\newcommand{\algfont}[1]{\mathsf{#1}}
\newcommand{\ermoracle}{\algfont{ERMOracle}}
\newcommand{\ghat}{\widehat{g}}
\newcommand{\rel}{\mathbf{Rel}_T}
\newcommand{\relT}[1]{\rel\left( \cF | x_1,y_1 \dots, x_{#1}, y_{#1} \right)}
\newcommand{\muhat}{\widehat{\mu}}
\newcommand{\gtil}{\widetilde{g}}
\newcommand{\nbrack}[2]{N_{[]}\left( #1, #2 \right)}
\DeclareMathOperator{\ber}{Ber}
\newcommand{\ftil}{\widetilde{f}}

\def\ddefloop#1{\ifx\ddefloop#1\else\ddef{#1}\expandafter\ddefloop\fi}
\def\ddef#1{\expandafter\def\csname bb#1\endcsname{\ensuremath{\mathbb{#1}}}}
\ddefloop ABCDEFGHIJKLMNOPQRSTUVWXYZ\ddefloop

\def\ddefloop#1{\ifx\ddefloop#1\else\ddef{#1}\expandafter\ddefloop\fi}
\def\ddef#1{\expandafter\def\csname fr#1\endcsname{\ensuremath{\mathfrak{#1}}}}
\ddefloop ABCDEFGHIJKLMNOPQRSTUVWXYZ\ddefloop

\def\ddefloop#1{\ifx\ddefloop#1\else\ddef{#1}\expandafter\ddefloop\fi}
\def\ddef#1{\expandafter\def\csname eul#1\endcsname{\ensuremath{\EuScript{#1}}}}
\ddefloop ABCDEFGHIJKLMNOPQRSTUVWXYZ\ddefloop

\def\ddefloop#1{\ifx\ddefloop#1\else\ddef{#1}\expandafter\ddefloop\fi}
\def\ddef#1{\expandafter\def\csname scr#1\endcsname{\ensuremath{\mathscr{#1}}}}
\ddefloop ABCDEFGHIJKLMNOPQRSTUVWXYZ\ddefloop

\def\ddefloop#1{\ifx\ddefloop#1\else\ddef{#1}\expandafter\ddefloop\fi}
\def\ddef#1{\expandafter\def\csname b#1\endcsname{\ensuremath{\mathbf{#1}}}}
\ddefloop ABCDEFGHIJKLMNOPQRSTUVWXYZ\ddefloop

\def\ddefloop#1{\ifx\ddefloop#1\else\ddef{#1}\expandafter\ddefloop\fi}
\def\ddef#1{\expandafter\def\csname bhat#1\endcsname{\ensuremath{\hat{\mathbf{#1}}}}}
\ddefloop ABCDEFGHIJKLMNOPQRSTUVWXYZ\ddefloop

\def\ddefloop#1{\ifx\ddefloop#1\else\ddef{#1}\expandafter\ddefloop\fi}
\def\ddef#1{\expandafter\def\csname btil#1\endcsname{\ensuremath{\tilde{\mathbf{#1}}}}}
\ddefloop ABCDEFGHIJKLMNOPQRSTUVWXYZ\ddefloop

\def\ddefloop#1{\ifx\ddefloop#1\else\ddef{#1}\expandafter\ddefloop\fi}
\def\ddef#1{\expandafter\def\csname bst#1\endcsname{\ensuremath{\mathbf{#1}^\star}}}
\ddefloop ABCDEFGHIJKLMNOPQRSTUVWXYZ\ddefloop

\def\ddefloop#1{\ifx\ddefloop#1\else\ddef{#1}\expandafter\ddefloop\fi}
\def\ddef#1{\expandafter\def\csname bst#1\endcsname{\ensuremath{\mathbf{#1}^\star}}}
\ddefloop abcdefghijklmnopqrstuvwxyz\ddefloop

\def\ddefloop#1{\ifx\ddefloop#1\else\ddef{#1}\expandafter\ddefloop\fi}
\def\ddef#1{\expandafter\def\csname b#1\endcsname{\ensuremath{\mathbf{#1}}}}
\ddefloop abcdefghijklnopqrstuvwxyz\ddefloop

\def\ddefloop#1{\ifx\ddefloop#1\else\ddef{#1}\expandafter\ddefloop\fi}
\def\ddef#1{\expandafter\def\csname barb#1\endcsname{\ensuremath{\bar{\mathbf{#1}}}}}
\ddefloop abcdefghijklmnopqrstuvwxyz\ddefloop

\def\ddef#1{\expandafter\def\csname c#1\endcsname{\ensuremath{\mathcal{#1}}}}
\ddefloop ABCDEFGHIJKLMNOPQRSTUVWXYZ\ddefloop
\def\ddef#1{\expandafter\def\csname h#1\endcsname{\ensuremath{\widehat{#1}}}}
\ddefloop ABCDEFGHIJKLMNOPQRSTUVWXYZ\ddefloop
\def\ddef#1{\expandafter\def\csname hc#1\endcsname{\ensuremath{\widehat{\mathcal{#1}}}}}
\ddefloop ABCDEFGHIJKLMNOPQRSTUVWXYZ\ddefloop
\def\ddef#1{\expandafter\def\csname t#1\endcsname{\ensuremath{\widetilde{#1}}}}
\ddefloop ABCDEFGHIJKLMNOPQRSTUVWXYZ\ddefloop
\def\ddef#1{\expandafter\def\csname tc#1\endcsname{\ensuremath{\widetilde{\mathcal{#1}}}}}
\ddefloop ABCDEFGHIJKLMNOPQRSTUVWXYZ\ddefloop

\usepackage{autonum}

\title{The Sample Complexity of Approximate Rejection Sampling with Applications to Smoothed Online Learning}
\author{Adam Block and Yury Polyanskiy \\ MIT}
\date{}

\begin{document}

\maketitle
\begin{abstract}
    Suppose we are given access to $n$ independent samples from distribution $\mu$
and we wish to output one of them with the goal of making the output
distributed as close as possible to a target distribution $\nu$.  In this work
we show that the optimal total variation distance as a function of $n$ is given
by $\tilde\Theta(\frac{D}{f'(n)})$ over the class of all pairs $\nu,\mu$ with a bounded $f$-divergence $D_f(\nu\|\mu)\leq D$. Previously, this question was studied only for the case when the Radon-Nikodym derivative of $\nu$ with respect to
$\mu$ is uniformly bounded. We then consider an application in the
seemingly very different field of smoothed online learning, where we show that recent
results on the minimax regret and the regret of oracle-efficient algorithms
still hold even under relaxed constraints on the adversary (to have bounded $f$-divergence, as opposed to bounded Radon-Nikodym derivative).  
Finally, we also study efficacy of importance sampling for mean estimates uniform
over a function class and compare importance sampling with rejection
sampling.

\end{abstract}

\long\def\nbyp#1{\textcolor{red}{\textbf{YP:} #1}}

\section{Introduction}

Consider the following problem: given $n$ independent samples from some base distribution $\mu$,
how can a learner generate a single sample from a target distribution $\nu$? This simple question
dates back decades, with the first formal solution, rejection sampling, provided already by \citet{von195113}.  Due to its simplicity, this sampling problem appears as a primitive in numerous applications in machine learning, theoretical computer science, and cryptography \citep{lyubashevsky2012lattice,liu1996metropolized,naesseth2017reparameterization,ozols2013quantum}; thus, constructing efficient solutions has filled many works \citep{grover2018variational,gilks1992adaptive,martino2011generalization}.  Perhaps surprisingly, though, the original solution of rejection sampling \citep{von195113} remains a popular method even today.

Given $X_1, \dots, X_n \sim \mu$, recall that rejection sampling takes as a parameter some $M$, which is a uniform upper bound on
the Radon-Nikodym derivative $\frac{d \nu}{d \mu}$, and for each $1 \leq i \leq n$, accepts $X_i$
with probability $\frac{1}{M} \cdot \frac{d \nu}{d \mu}(X_i)$ and returns an arbitrary accepted $X_i$ as a sample from $\nu$.  It is an easy exercise to see that
if $M \geq \norm{\frac{d \nu}{d \mu}}_{\infty}$, then any accepted sample has distribution $\nu$.
Furthermore, it is not hard to see that any sample gets accepted with probability $\frac 1M$
independently of other samples and thus, if we want to have at least one accepted sample with high
probability, we require $n = \Theta(M)$.  While there has been quite a lot of work in the
information theory community dedicated to refining this bound
\citep{liu2018rejection,harsha2007communication} as well as developments in the statistical
community dedicated to improving sampling efficiency under strong structural assumptions
\citep{gilks1992adaptive,gorur2011concave}, the scope of most all of this work is limited by the
requirement that $\norm{\frac{d \nu}{d \mu}}_{\infty} < \infty$.  In many settings,
this assumption is false \citep{block2023smoothed}; as a result, we focus on a similar problem without the
stringent assumption on a uniform upper bound.  Unfortunately, it is not hard to see that there
exist examples where we simply cannot recover a sample \textit{exactly} from $\nu$  without
this uniform upper bound (see \Cref{thm:lowerboundlinearformal} for an example). Consequently, we relax
our desideratum to consider \textit{approximate} sampling.  Specifically, we ask the following question:
\begin{quotation}
    \textit{How many independent samples $X_1, \dots, X_n$ do we need from a source distribution
    $\mu$ such that we can select some $\jstar \in [n]$ in order for the law of $X_{\jstar}$ to be
    within total variation distance $\epsilon$ of $\nu$?}
\end{quotation}
Despite its simplicity, to the best of our knowledge this question has not been considered in the literature to date.  We emphasize several special cases in the related work in \Cref{app:relatedwork}.  In this work we give a complete answer to this question with essentially matching upper and lower bounds for all superlinear $f$-divergences of practical interest.  While the upper bounds are achieved with a modified rejection sampler and the analysis follows without too much difficulty from classical work, the lower bounds require a more technical approach.  In order to quantify how far apart $\nu$ is from $\mu$, we use the information-theoretic notion of an $f$-divergence, where for two measures $\nu \ll \mu$ defined on a common set and a convex function $f$, we define
\begin{align}
    \divf{\nu}{\mu} = \ee_\mu\left[ f\left( \frac{d\nu}{d\mu}(Z) \right) \right].
\end{align}
We give a more formal definition below, but we observe here that the notion of $f$-divergence generalizes common divergences including total variation, KL-divergence, Renyi divergences, and $\cE_\gamma$ divergence \citep{polyanskiy2022,van2014renyi,asoodeh2021local}.  We will make the assumption that for some convex $f$, the source and target measures satisfy $\divf{\nu}{\mu} < \infty$ and ask what the sample complexity of $\epsilon$-approximate rejection sampling is under this constraint.  Interestingly, the answer depends on the tail behavior of $f$; in particular, if $\sup f'(x) < \infty$ then rejection sampling cannot work under only this constraint (see Proposition \ref{prop:lowerboundlinearinformal}).  If we have an $f$-divergence constraint with $f'(\infty) = \infty$, however, we will see that
\begin{align}
    n = \widetilde \Theta\left( (f')^{-1}\left( \frac{8 \cdot \divf{\nu}{\mu}}{\epsilon} \right)  \right)
\end{align}
samples is both necessary and sufficient in order to generate a sample $X_{\jstar}$ that is $\epsilon$-close in total variation.  In fact, we show that von Neumann's original rejection sampler is essentially optimal for this problem and we do not require the more complicated samplers introduced for exact sampling by \citet{harsha2007communication,liu2018rejection}.  As mentioned above, the upper bounds are relatively standard, with much of the technical effort involving the construction of lower bounds.

While the above results are interesting in their own right, we emphasize one
key use case of our results in a seemingly unrelated field: \textit{smoothed online learning}.  We briefly recall the setup.  For general online learning, we fix a set of contexts $\cX$, a set of targets $\cY$ and a function class $\F: \cX \to \cY$ as well as a loss function $\ell: \cY \times \cY \to [0,1]$.  For some horizon $T$, online learning proceeds in rounds where for each time $1 \leq t \leq T$ the following happens:
\begin{enumerate}
    \item Nature chooses some context $x_t$ and label $y_t$.
    \item the Learner chooses some prediction $\yhat_t \in \cY$.
    \item The learner sees $y_t$ and suffers loss $\ell(\yhat_t, y_t)$.
\end{enumerate}
As in \citet{block2022smoothed,haghtalab2022oracle}, we distinguish between the \emph{proper} and \emph{improper} settings.  In the former, the Learner must choose some function $f_t \in \F$ before seeing $x_t$ and then predicts $\yhat_t = f_t(x_t)$.  In the latter, the Learner observes $x_t$ and then predicts an arbitrary $\yhat_t \in \cY_t$.  The goal in both cases is to minimize the expected regret to the best function in hindsight, where
\begin{align}
    \reg_T = \sum_{t = 1}^T \ell(\yhat_t, y_t) - \inf_{f \in \cF} \sum_{t = 1}^T \ell(f(x_t), y_t).
\end{align}
As stated, there is no restriction on Nature's choice of the context and label, which is called the adversarial setting.  Despite its popularity due to the robustness of the regime and the lack of assumptions, there are two major problems with the fully adversarial setting: first, simple function classes like thresholds in one dimension that can be easily learned when the data appear independently become unlearnable in the adversarial regime \citep{rakhlin2015sequential,littlestone1988learning}; second, even when function classes are learnable, they often cannot be learned efficiently \citep{hazan2016computational}.  In order to solve the first issue, the notion of smoothed online learning has recently gained traction \citep{rakhlin2011online,haghtalab2022oracle,haghtalab2022smoothed,block2022smoothed,block2022efficient}.  Motivated by smoothed analysis of algorithms, \citet{rakhlin2011online,haghtalab2022smoothed} consider the following setting.  For a fixed base measure $\mu$ on some set $\cX$, we say that a measure $\nu$ is $\sigma$-\emph{smooth} with respect to $\mu$ if $\norm{\frac{d\nu}{d\mu}}_\infty \leq \frac 1\sigma$.  An adversary is $\sigma$-\emph{smooth} with respect to some fixed $\mu$ if for all $t$, it holds that the distribution $p_t$ of $x_t$ conditioned on all the history is $\sigma$-smooth.  One motivation for this definition is to suppose that Nature is fully adversarial, but corrupted by some small amount of noise.  For example, if $\cX = \rr^d$, we could imagine adding a small amount of uniform or Gaussian noise to an adversarial input \citep{block2023smoothed}.  In \citet{block2022smoothed,haghtalab2022smoothed}, the minimax optimal rates for smoothed online learning were derived up to polylogarithmic factors.  As an example, if we let $\vc{\cF}$ denote the Vapnik-Chervonenkis dimension \citep{blumer1989learnability} of some binary valued function class $\F$, then there exists some algorithm capable of achieving, with respect to the indicator loss,
\begin{align}
    \ee\left[ \reg_T \right] = O\left( \sqrt{T \cdot \vc{\cF} \cdot \log\left( \frac{T}{\sigma} \right)} \right).
\end{align}
Unfortunately, in many common settings, a uniform bound on $\frac{d p_t}{d\mu}$ may not be achievable.  For example, consider again the case of a small amount of Gaussian noise in $\rr^d$ being added to an adversarial input.  A natural choice of $\mu$ would be some fixed Gaussian, but there is no way to ensure that $\norm{\frac{d p_t}{d\mu}}_\infty$ is finite.  Even when the Radon-Nikodym derivative is uniformly bounded, it may be, as in many high dimensional settings, that this bound is too large for the resulting implications to be meaningful.  Thus, in \Cref{sec:smoothedonline}, we propose a more general notion, of an $f$-smoothed adversary, where the distribution $p_t$ of the contexts $x_t$ conditional on the history satisfies $\divf{p_t}{\mu} \leq \frac 1\sigma$.  In this harder setting, the results of \citet{block2022smoothed,haghtalab2022oracle,haghtalab2022smoothed} no longer apply due to the breakdown of a key technical step used in the proofs of all of these results.  In \Cref{sec:smoothedonline}, we apply our bounds on the sample complexity of approximate rejection sampling to generalize the approach of these works and achieve upper bounds on the information theoretic rates of  $f$-smoothed online learning, which are tight for some $f$-divergences.

While the information theoretic rates provided in \citet{block2022smoothed,haghtalab2022smoothed} are important, the suggested algorithms that achieve these rates are computationally intractable and thus two \emph{oracle-efficient} algorithms were also proposed, where the learner has access to an Empirical Risk Minimization (ERM) oracle returning the minimizer over $\cF$ of a weighted cumulative loss function evaluated on some data set (see Definition \ref{def:ermoracle} for a formal definition).  Once again, the analysis of these two algorithms does not extend beyond the standard smoothed setting; in \Cref{sec:smoothedonline}, we again apply our rejection sampling sample complexity bounds to demonstrate that, by modifying the hyperparameters of the two proposed algorithms, we can still maintain a no-regret guarantee under the significantly more general $f$-smoothed online learning setting.  

We defer discussion of related work to \Cref{app:relatedwork} for the sake of space.  We now summarize our key contributions:
\begin{itemize}
    \item In \Cref{thm:upperbound}, we provide an upper bound on the complexity of approximately sampling from some target measure $\nu$ given access to samples from $\mu$.  In particular, we show that by modifying classical rejection sampling, $\widetilde{\Theta}\left( (f')^{-1}\left( \frac{\divf{\nu}{\mu}}{\epsilon} \right) \right)$ samples suffice to obtain a sample with total variation distance at most $\epsilon$ from the target.
    \item In Proposition \ref{prop:lowerboundlinearinformal} and \Cref{thm:lowerbound,thm:lowerboundweakf}, we show that the upper bound given by rejection sampling is essentially tight.  In particular, we show that rejection sampling is in some sense generic in that ``the best'' way to use samples from $\mu$ to produce a sample from $\nu$ is the approach described above.  Furthermore, we show that if $f'$ is bounded above, then the approximate sampling problem is impossible; if $f'$ is unbounded, we show in \Cref{thm:lowerbound,thm:lowerboundweakf} that the sample complexity derived in \Cref{thm:upperbound} is essentially tight as $\epsilon \downarrow 0$.  In particular, Theorem \ref{thm:lowerbound} shows that for all $n$, there exist distributions with bounded $f$-divergence such that $\Omega\left( (f')^{-1}\left( \frac{\divf{\nu}{\mu}}{\epsilon} \right) \right)$ samples are necessary to produce an $\epsilon$-approximate sample from the target measure, while in Theorem \ref{thm:lowerboundweakf}, we show that (for a slightly smaller class of $f$ satisfying a mild growth condition) there exist distributions such that the preceding lower bound holds uniformly in $n$.
    \item In \Cref{sec:smoothedonline}, we generalize previous results on smoothed online learning to the significantly more general setting of $f$-smoothed online learning.  In particular, we derive minimax upper bounds without regard to computation time as well as demonstrating that two oracle-efficient algorithms (one proper and one improper) proposed in \citet{block2022smoothed} remain no-regret even in the more general $f$-smoothed online learning setting.  Moreover, in \Cref{thm:properinformal}, we answer an open question in \citet{block2022smoothed} by showing that an instance of FTPL has regret scaling like $\sigma^{-1/4}$ as opposed to $\sigma^{-1/2}$, where $\sigma$ is the smoothness parameter of the adversary; this generalizes a result of \citet{haghtalab2022oracle} to arbitrary context spaces.
    \item In \Cref{sec:importance}, we prove new bounds on the quality of importance sampling for estimating means with respect to a target $\nu$ uniformly over a function class $\cF$ when we have access to samples from $\mu$.  We then compare these results to estimates using rejection sampling assuming $\divf{\nu}{\mu} < \infty$ for the special case of $\chi^2$-divergence and compare these results with earlier bounds from \citet{chatterjee2018sample,cortes2010learning}.
\end{itemize}

\paragraph{Notation} In the sequel, we will always denote by $\mu$ a base measure on the set $\cX$ with associated $\sigma$-algebra $\scrF$.  We will denote by $X_{1:n} = (X_1, \dots, X_n)$ a vector of $n$ independent samples from $\mu$ and we will let $\jstar$ be a selection rule.  We will reserve $\nu$ for our target measure and the letters $\epsilon, \delta, \gamma$ will all be reserved for small positive real constants.  Furthermore, we will reserve $f$ for a convex function mapping the positive reals to the positive reals satisfying $f(1) = f'(1) = 0$.  Furthermore, for such an $f$, we will let $f^{-1}(u) = \inf\left\{ t > 0 | f'(t) \geq u \right\}$ where we adopt the standard convention of taking the infimum of the empty set to be infinite.  For a given random variable $Y$, we will denote by $P_Y$ the distribution of $Y$.  We use $O(\cdot), \Omega(\cdot)$ to denote asymptotic big-oh notation and apply tildes to hide polylogarithmic factors.

\section{Problem Setup and Notation}\label{sec:prelims}
In this section, we formally define the necessary information theoretic quantities and state the problem.  To begin, we define $f$-divergence.  For more information on information theoretic notions, see \citet{polyanskiy2022}
\begin{definition}\label{def:fdivergence}
    Let $f: [0, \infty] \to \rr_{\geq 0} \cup \{\infty \}$ be a convex function satisfying $f(1) = f'(1) = 0$.  For two probability measures $\nu, \mu$ on some space $\cX$, define the $f$-divergence,
    \begin{align}
        \divf{\nu}{\mu} = \ee_\mu\left[ f\left( \frac{d\nu}{d \mu}(Z) \right) \I\left[ \frac{d \nu}{d \mu}(Z) < \infty \right] \right] + f'(\infty) \mu\left( \frac{d \nu}{d \mu}(Z) = \infty \right).
    \end{align}
    Note that if $\nu \ll \mu$ then we may ignore the second term.
\end{definition}
\begin{remark}
    As a technical aside, throughout the paper, we will be using $f'$ and $f''$ to denote the first and second derivatives of the $f$ appearing in Definition \ref{def:fdivergence}.  By Rademacher's Theorem \citep{rademacher1919partielle}, $f$ is differentiable almost everywhere, but for any points where $f$ is not differentiable, we will take $f'$ to be the maximal subgradient.  As $f$ is increasing, $f'$ is nondecreasing and thus we can take $f''$ to be the right derivative of $f'$, which is always well-defined.
\end{remark}
We will phrase our results in terms of $\divf{\nu}{\mu}$ for general $f$, but there are several important examples that will come up throughout the paper.  Before formally introducing the problem, we will give several examples of well-known $f$-divergences:
\begin{example}[Total Variation]\label{ex:tv}
    Consider $f(x) = \abs{x - 1} - (x - 1)$.  In this case we have
    \begin{align}
        \divf{\nu}{\mu} = \tv(\nu, \mu) = \sup_{A \in \scrF} \abs{\nu(A) - \mu(A)}
    \end{align}
    the total variation distance, where $\scrF$ is the common $\sigma$-algebra over $\cX$ on which $\nu, \mu$ are defined.
\end{example}
\begin{example}[KL Divergence] \label{ex:kl}
    If we set $f(x) = x \log(x) - x + 1$ then we get $\divf{\nu}{\mu} = \dkl{\nu}{\mu}$ the KL divergence.
\end{example}
\begin{example}[Renyi Divergence]\label{ex:renyi}
    If we set $f(x) = x^\lambda - \lambda x + \lambda - 1$, then we get that
    \begin{align}
        \divf{\nu}{\mu} = e^{(\lambda - 1) \dren{\nu}{\mu}}
    \end{align}
    where $\dren{\nu}{\mu}$ is the Renyi divergence of order $\lambda$.  Special cases of $\dren{\nu}{\mu}$ include the case where $\lambda \downarrow 1$ in which case we have the KL divergence again and $\lambda = 2$, in which case we recover (a monotone transformation of) the standard $\chi^2$ divergence.
\end{example}
\begin{example}[$\cE_\gamma$ Divergence]\label{ex:Egamma}
    If, for $\gamma \geq 1$, we set $f_\gamma(x) = (x - \gamma)_+$, then denote by $\cE_\gamma(\nu
    || \mu)$ the divergence associated with this $f$.  This divergence was originally defined
    in~\cite[(2.141)]{YP10} for the study of channel coding. Since then it appeared prominently in
    the study of differential privacy \citep{asoodeh2021local} and wiretap channels
    \citep{liu2017Egamma}.  It will also be crucial in the proof of our lower bounds below.
\end{example}
We now define the primary object of study.  Given $X_{1:n} = (X_1, \dots, X_n)$ a tuple of elements of $\cX$, we define a \emph{selection rule} $\jstar$ as any random variable taking values in $[n]$ and depending in any way on $X_{1:n}$.  We are now ready to formally state the main problem:
\begin{quote}
    \textbf{Question:} Suppose that $\cX$ is an arbitrary set with $\sigma$-algebra $\scrF$ and suppose that $\mu, \nu$ are probability measures with respect to $\scrF$ satisfying, for some fixed $f$, $\divf{\nu}{\mu} < \infty$.  For fixed $\epsilon > 0$, how large does $n$ have to be such that there exists a selection rule $\jstar$ ensuring that $\tv\left( P_{X_{\jstar}}, \nu \right) < \epsilon$?
\end{quote}
As an example, we consider traditional rejection sampling.  We construct a random set $\cS \subset [n]$ by adding $j$ to $\cS$ with probability $\frac 1M \cdot \frac{d \nu}{d\mu}(X_j)$, which is at most $1$ by the assumption that $M > \norm{\frac{d \nu}{d\mu}}_\infty$.  If $\cS$ is nonempty, we let $\jstar$ be an arbitrary element and otherwise we select $\jstar$ uniformly at random.  As we shall show for the sake of completeness (see Lemma \ref{lem:rejsampling} in \Cref{app:rejsampling}), the probability that $\cS$ is empty is at most $e^{-\frac{n}{M}}$ and if $\cS$ is nonempty then $X_{\jstar}$ is distributed according to $\nu$.  Thus if $n = M \log\left( \frac 1\delta \right)$, with probability at least $1 - \delta$, $X_{\jstar}$ is distributed according to $\nu$.  Because we required $M > \norm{\frac{d \nu}{d \mu}}_{\infty}$, we see that $\Theta\left( \norm{\frac{d \nu}{d \mu}}_{\infty} \log\left( \frac 1\delta \right) \right)$ samples are sufficient to exactly sample from $\nu$ with high probability.  The necessity will be seen as a very special case of our lower bounds in the following section.

\section{Sample Complexity of Rejection Sampling}\label{sec:rejsampling}

In this section, we state and sketch the proofs of our main results regarding rejection sampling and fully answer the question raised in \Cref{sec:prelims}.  We will divide our results into two theorems, one providing an upper bound using a modified version of rejection sampling, and the other giving an almost matching lower bound.  We begin with the upper bound:
\begin{theorem}[Upper Bound]\label{thm:upperbound}
    Suppose that $\mu, \nu$ are probability distributions on some set $\cX$ and suppose that $X_1, \dots, X_n \sim \mu$ are independent.  Fix some $f$ satisfying the conditions in Definition \ref{def:fdivergence}.  For $\epsilon > 0$, if
    \begin{align}\label{eq:upperbound}
        n \geq \frac{2}{1-\epsilon} \log\left( \frac{2}{\epsilon} \right) (f')^{-1}\left( \frac{4 \cdot \divf{\nu}{\mu}}{\epsilon} \right) \vee 2
    \end{align}
    then there exists a selection rule $\jstar$ satisfying $\tv\left( P_{X_{\jstar}}, \nu \right) \leq \epsilon$.
\end{theorem} 

We will split our discussion into two cases: the superlinear case, where $f'(t) \uparrow \infty$ as $t \uparrow \infty$ and the linear case, where $f'(t)$ is bounded from above.  In the former, we will see that as $n \uparrow \infty$, we can always use rejection sampling to get an increasingly good approximation of a sample from $\nu$ because $(f')^{-1}$ is finite on the entire positive real line.  In the linear case, however, we shall shortly prove that no selection rule can hope to get arbitrarily close to $\nu$ in total variation.  Before sketching the proof of \Cref{thm:upperbound}, we provide some examples.
\begin{example}[Total Variation]\label{ex:tv2}
    Recall that total variation is the $f$-divergence such that $f(x) = \abs{x - 1} - x + 1$.  Note that $f'(x) = 0$ for all $x > 1$ and so $(f')^{-1}(M)$ is infinite for $M > 0$.  Thus \Cref{thm:upperbound} is vacuous when we only have control over total variation, as expected.
\end{example}
\begin{example}[KL Divergence]\label{ex:kl2}
    As we saw in \Cref{ex:kl}, KL divergence is the $f$-divergence where we set $f(x) = x \log(x) - x + 1$.  In this case, we see that $f'(x) = \log(x)$ and so \Cref{thm:upperbound} tells us that in order to be $\epsilon$-close in total variation, $\widetilde{O}\left( \exp\left( \frac{\dkl{\nu}{\mu}}{\epsilon} \right) \right)$ samples suffice.
\end{example}
\begin{example}[Renyi Divergence]\label{ex:renyi2}
    Remember from \Cref{ex:renyi} that $f(x) = x^\lambda - \lambda x + \lambda - 1$ for $\lambda > 1$ defines the Renyi divergence.  In this case we see that $\widetilde{O}\left( e^{\dren{\nu}{\mu}} \epsilon^{- \frac{1}{\lambda - 1}} \right)$ samples suffice.  As $\lambda \uparrow \infty$, we recover the standard rejection sampling bound by taking $\epsilon \downarrow 0$ and noting that $D_\infty(\nu || \mu) = \norm{\frac{d\nu}{d\mu}}_\infty$.  In the special case of $\lambda = 2$, we note that Renyi divergence recovers $\chi^2$-divergence and note that $\widetilde{O}\left( \frac{\chi^2(\nu ||\mu)}{\epsilon} \right)$ samples suffice.
\end{example}
We now sketch the proof of the upper bound, deferring details to \Cref{app:rejsampling}:
\begin{proof}[Proof of \Cref{thm:upperbound}]
    Let $\nuM$ denote the measure $\nu$ conditioned on the event that $\frac{d \nu}{d\mu} \leq M$ and let $\nutil$ denote the law of the sample produced by rejection sampling from $\nuM$ with $n$ samples.  The standard analysis of rejection sampling tells us that if $n = \Omega\left( \log\left( \frac{1}{\epsilon} \right)M \right)$ then $\tv\left( \nutil, \nuM \right) \leq \epsilon$.  We show in Lemma \ref{lem:expectedrestrictedrdbound} that if $M > 1$, then
    \begin{align}
        \nu\left( \frac{d\nu}{d\mu} > M \right) \leq \frac{2 \cdot \divf{\nu}{\mu}}{f'(M/2)}.
    \end{align}
    Using this result, we show that $\tv\left( \nuM, \nu \right) \leq \frac{2 \cdot \divf{\nu}{\mu}}{f'(M/2)}$ and conclude by applying the triangle inequality.
\end{proof}

We now turn to our lower bounds.  In particular, we show that for any $f$-divergence, there exist distributions $\mu, \nu$ satisfying $\divf{\nu}{\mu} < \infty$ such that in order for there to exist a selector rule guaranteeing that $\tv\left( P_{X_{\jstar}}, \nu \right) < \epsilon$, we require $n$ to be sufficiently large.  We will again split our discussion into the linear and superlinear cases.  For the linear case, we have the following lower bound:
\begin{proposition}[Lower Bound, Linear Case]\label{prop:lowerboundlinearinformal}
    Suppose that $f$ is a convex function as in Definition \ref{def:fdivergence} satisfying $f'(t) \leq C < \infty$ for all $t > 1$.  Then there exist distributions $\mu, \nu$ such that $\divf{\nu}{\mu} < \infty$ and $\epsilon = \epsilon(f, \divf{\nu}{\mu}) > 0$ such that for all $n$ and  $X_1,\ldots X_n \stackrel{iid}{\sim}\mu$.
    \begin{align}
        \inf_{\jstar} \tv\left( P_{X_{\jstar}}, \nu \right) \geq \epsilon
    \end{align}
     where the infimum is over all selection rules $\jstar$.
\end{proposition}
Note that Proposition \ref{prop:lowerboundlinearinformal} matches the upper bound for linear $f$ in \Cref{thm:upperbound} and reflect the fact that for $f$ that do not grow superlinearly, $\divf{\nu}{\mu} < \infty$ provides very weak control on $\nu$.  Intuitively this should be clear: note that if $f$ is in the linear regime, then $\divf{\nu}{\mu}$ can remain finite even when $\nu$ is singular with respect to $\mu$ and thus samples from $\mu$ can never hope to approximate $\nu$ to arbitrary precision.  A full proof can be found in \Cref{app:rejsampling}.

Moving on to the more interesting case of superlinear $f$, we provide a lower bound that matches the upper bound found in \Cref{thm:upperbound} for all superlinear $f$.  
\begin{theorem}[Lower Bound, Superlinear Case]\label{thm:lowerbound}
    Let $f$ be a convex function as in Definition \ref{def:fdivergence} that grows superlinearly.  Then for all $0 < \epsilon \leq 1/4$ and $\delta > 2 f(1/2)$, there exists a pair of measures $\nu, \mu$ such that $\divf{\nu}{\mu} \leq \delta$ and any selection rule $\jstar$ satisfying $\tv(P_{X_{\jstar}}, \nu) \leq \epsilon$ requires
    \begin{align}\label{eq:lowerboundonen}
        n \geq \frac 12 \cdot (f')^{-1}\left( \frac{\delta}{2\epsilon} \right).
    \end{align}
\end{theorem}
While we provide full details in Appendix \ref{app:rejsampling}, we provide a sketch of the proof here:
\begin{proof}
    A simple computation found in Lemma \ref{lem:jstarrdboundedbyn} tells us that if $\nutil$ is the law of $X_{\jstar}$, then the Radon-Nikodym derivative of $\nutil$ with respect to $\mu$ is uniformly bounded by $n$.  Another computation, found in Lemma \ref{lem:egammalowerbound} tells us that if $\nutil$ has likelihood ratio bounded by $n$, then we can lower bound $\tv\left( \nutil, \nu \right)$ by $\cE_n(\nu || \mu)$.  Combining these facts, we see that it suffices to exhibit two distributions $\mu, \nu$, such that $\divf{\nu}{\mu} \leq \delta$ and $\cE_n\left( \nu || \mu \right) \geq \epsilon$ for all $n$ not satisfying \eqref{eq:lowerboundonen}.  Thus, we have reduced the proof to determining if the point $(\epsilon, \delta)$ lies above some point in the \emph{joint range} of $\mu$ and $\nu$, i.e., the set $\left\{ (\cE_n(\nu || \mu), \divf{\nu}{\mu}) \right\}$ where $\mu$ and $\nu$ vary over all distributions.  In \citet{harremoes2011pairs}, it was shown that the distributions extremizing the joint range are typically pairs of Bernoulli random variables.  We thus consider $\mu = \ber\left( \frac{\epsilon}{n} \right)$ and $\nu = \ber(2 \epsilon)$ and show that $\cE_n(\nu || \mu) = \epsilon$, while $\divf{\nu}{ \mu} \leq \delta$, unless $n$ is sufficiently large so as to satisfy \eqref{eq:lowerboundonen}.  The result follows.
\end{proof}
Note that Theorem \ref{thm:lowerbound} tells us that, up to logarithmic factors, the sample complexity determined in Theorem \ref{thm:upperbound} is optimal.  There is one disadvantage to the above result, however: as is clear from the proof, the distributions $\mu$ and $\nu$ depend on $n$ and thus the order of quantifiers in Theorem \ref{thm:lowerbound} is weaker than that in Theorem \ref{thm:upperbound}.  In order to address this shortcoming, we prove a slightly weaker lower bound under a mild growth condition on $f$:
\begin{theorem}\label{thm:lowerboundweakf}
    Let $f$ be a convex function as in Definition \ref{def:fdivergence} that grows superlinearly.  Suppose that $f$ satisfies a mild growth condition (see \Cref{thm:lowerboundsuperlinearformal} for formal statement).  Then, for any $\zeta > 0$, there exist distributions $\mu, \nu$ with $\divf{\nu}{\mu} < \infty$ such that for all sufficiently large $n \in \mathbb{N}$, with $X_1, \dots, X_n$ sampled independently from $\mu$, it holds that
    \begin{align}\label{eq:lowerboundweakf}
        \inf_{\jstar} \tv\left( P_{X_{\jstar}}, \nu \right) \geq \frac {\zeta^{1+\zeta}}8 \cdot \left( \frac{\divf{\nu}{\mu}}{f'(n)} \right)^{1 + \zeta}
    \end{align}
    where the infimum is taken over all selection rules. 
\end{theorem}
We note that the mild growth condition required in \Cref{thm:lowerboundweakf} is purely technical and likely could be removed with more elaborate analysis; on the other hand, this condition is satisfied by all commonly used, superlinear $f$-divergences of which we are aware.  By \Cref{thm:upperbound}, we see that if
\begin{equation}
   n = \widetilde{O}\left((f')^{-1}\left( \frac{\divf{\nu}{\mu}}{\epsilon} \right) \right),
\end{equation} 
then rejection sampling suffices to generate an $\epsilon$-approximate sample from $\nu$.  On the other hand, setting $\zeta = o(1)$ as $\epsilon \downarrow 0$, \Cref{thm:lowerboundweakf} tells us that in the worst case, we require
\begin{equation}
    n = \widetilde \Omega\left( (f')^{-1}\left( \frac{\divf{\nu}{\mu}}{\epsilon^{1- o(1)}} \right) \right)
\end{equation}
samples for the right hand side in \eqref{eq:lowerboundweakf} to be below $\epsilon$.  Thus, as $\epsilon \downarrow 0$, these bounds essentially match.  In particular, because the $f$-divergences in \Cref{ex:kl2,ex:renyi2} satisfy the mild growth condition, the sample complexity upper bounds derived in those examples are indeed tight for all sufficiently large $n$.

We defer a detailed proof of \Cref{thm:lowerboundweakf} to \Cref{app:rejsampling}.  The method is similar to that of \Cref{thm:lowerbound} in that we reduce to lower bounding $\cE_n(\nu || \mu)$ for distributions $\nu, \mu$ with bounded $f$-divergence.  The difference is that we exhibit a \emph{single} pair $(\mu, \nu)$, depending on $f$ but independent of $n$, such that the desired properties hold.

Combining \Cref{thm:upperbound,thm:lowerbound,thm:lowerboundweakf}, we have shown that $\widetilde \Theta\left( (f')^{-1}(\divf{\nu}{\mu}) / \epsilon \right)$ samples are both necessary and sufficient to generate an $\epsilon$-approximate sample from $\nu$.  One immediate application of these results is to the problem of estimating means according to $\nu$ uniformly over some function class $\cF$ when given samples from $\mu$.  In \Cref{sec:importance}, we compare estimators using \Cref{thm:upperbound} to the classical importance sampling approach.  For the sake of space, this is deferred to the appendix; we now proceed to our main application regarding smoothed online learning.

\section{Smoothed Online Learning}\label{sec:smoothedonline}
Our most important immediate application is to the question of generalizing smoothed online learning as outlined in the introduction.
In this section, we extend results proved for smoothed adversaries \citep{rakhlin2011online,block2022smoothed,haghtalab2022smoothed,haghtalab2022oracle} described in the introduction to allow for a more powerful Nature.  To do this, we employ the following definition:
\begin{definition}\label{def:weaksmooth}
    Fix a base measure $\mu$ on some set $\cX$.  We say that a measure $\nu$ is $(f,\sigma)$-smooth (or $f$-smooth) with respect to $\mu$ if $\divf{\nu}{\mu} \leq \frac 1\sigma$.  An adversary is $(f, \sigma)$-smooth with respect to $\mu$ if for all $1 \leq t \leq T$, the distribution $p_t$ of $x_t$, conditioned on all the history, is $(f, \sigma)$-smooth.
\end{definition}
Definition \ref{def:weaksmooth} motivates an obvious question: can we achieve improvement over the fully adversarial setting even when we only require Nature to be $f$-smooth?  The answer will, of course, depend on what $f$ we choose.  For the case of eventually linear $f$, for example, we see that no improvement is possible in general:
\begin{proposition}\label{prop:tvlowerbound}
    Suppose that $\cF = \left\{ x \mapsto \I[x \geq \theta] | \theta \in [0,1] \right\}$ is the class of thresholds in one dimension.  Let $f$ be a convex function as in Definition \ref{def:fdivergence} that is eventually linear, in the sense that $f'$ is bounded above.  For all $0 < \sigma < 1$ there is a $(f, \sigma)$-smooth adversary such that any learner experiences $\ee\left[ \reg_T \right] = \Omega(T)$.
\end{proposition}
This result, proved in \Cref{app:smoothedonline}, is not surprising in light of the fact that fully adversarial online learning of $\cF$ is impossible; if $f$ is linear, then Nature can mix the worst-case adversary with a base distribution and still incur linear regret with finite $\divf{\nu}{\mu}$.  More interesting is the case of stronger $f$-divergences.  Before we present our results, we state our main technical tool, which generalizes a technique introduced in \citet{haghtalab2022smoothed} and extended in \citet{block2022smoothed}.  In those works, the authors introduced a coupling between the sequence contexts produced by a smooth, adaptive adversary and a larger set of independent sampled drawn from the base measure.  Using the tools developed in \Cref{sec:rejsampling}, we extend this technique beyond the case of uniformly bounded Radon-Nikodym derivatives:
\begin{lemma}\label{lem:coupling}
    Let $\cX$ be a set and $\mu$ some measure on $\cX$.  Suppose that an adversary is $(f,\sigma)$-smooth with respect to $\mu$ for some $f$ satisfying the conditions of Definition \ref{def:fdivergence} such that $\sup f'(t) = \infty$.  For any $T$ and any $\epsilon, \delta > 0$, if
    \begin{align}
        n \geq \frac{1}{1 - \epsilon} \log\left( \frac{T}{\delta} \right)(f')^{-1}\left( \frac{1}{\epsilon \sigma} \right)
    \end{align}
    then there exists a coupling between $(x_1, \dots, x_T)$ and $\left\{ Z_{t,j} | 1 \leq t \leq T \text{ and } 1 \leq j \leq n \right\}$ such that the $(x_1, \dots, x_T)$ are distributed according to the adversary, the $Z_{t,j} \sim \mu$ are independent, and, with probability at least $1 - \delta$, there are selection rules $\jstar_t$ such that $\tv\left( P_{x_t}, P_{Z_{t, \jstar_t}} \right) \leq \epsilon$.
\end{lemma}
We defer the construction of the coupling to \Cref{app:smoothedonline}; for now we focus on the implications.  Our first result extends \citet[Theorem 3]{block2022smoothed} and \citet[Theorem 3.1]{haghtalab2022smoothed} to the case of  $f$-smoothed online learning.  While we state the result for general real-valued function classes in \Cref{app:smoothedonline}, for the sake of simplicity we restrict our focus to binary-valued $\F$ here.
\begin{theorem}\label{thm:minimax}
    Suppose $\cF \to \left\{ \pm 1 \right\}$ is a binary valued function class and let $\vc{\cF}$ denote its Vapnik-Chervonenkis dimension.  Suppose that $(x_t, y_t)$ are generated by a $(f, \sigma)$-smoothed adversary in the sense of Definition \ref{def:weaksmooth} such that $f'(\infty) = \infty$.  Then there exists an algorithm such that
    \begin{align}
        \ee\left[ \reg_T \right] \lesssim \sqrt{T \log(T) \cdot \vc{\cF} } +  \inf_{0 < \epsilon < 1} \epsilon T + \sqrt{T  \vc{\cF} \log\left(T (f')^{-1}\left( \frac{1}{\epsilon \sigma} \right) \right) }.
    \end{align}
\end{theorem}
We remark that \Cref{thm:minimax} is a special case of the more general \Cref{thm:minimaxgeneral} applying to arbitrary real-valued function classes, which we state and prove in \Cref{app:smoothedonline}.  The proof follows the approach of \citet{block2022smoothed} with the modification of applying the more general coupling in Lemma \ref{lem:coupling} and is deferred to the appendix.  Here, we consider two instantiations of $f$-divergences.  First, for the case of Renyi divergence (see \cref{ex:renyi,ex:renyi2}), we see that for a Renyi-smoothed adversary, regret of the order $\widetilde{O}\left( \left( 1 + \frac 1{\lambda - 1} \right)\sqrt{T \vc{\cF} \log\left( \frac 1\sigma \right)} \right)$ is attainable.  Observe that when $\lambda \uparrow \infty$, we recover the results of \citet{block2022smoothed}.  On the other hand, if $\lambda$ is bounded away from 1, which covers the case of an adversary bounded in $\chi^2$ divergence, we see that the cost of assuming only $\dren{p_t}{\mu} < \infty$ is only on the order of a constant more than in the standard setting.  The situation is different if we assume that the adversary is $f$-smoothed in the sense of KL divergence: in this case, we are only able to recover regret scaling like $\widetilde{O}\left( T^{2/3}\left( \vc{\cF} / \sigma \right)^{1/3} \right)$.  While the results for Renyi divergence are optimal up to polylogarithmic factors, we leave as an interesting open direction the question of whether the regret against a KL-smoothed adversary can be improved.

While \Cref{thm:minimax} is important insofar as it gives the information theoretic rates of $f$-smoothed online learning, the algorithms, where provided, are computationally intractable.  We now demonstrate that two algorithms proposed by \citet{block2022smoothed,haghtalab2022oracle} for smoothed online learning remain no-regret even if we weaken our assumptions to include $(f,\sigma)$-smoothed adveraries.  These algorithms are \emph{oracle-efficient}, i.e., they make few calls to an Empirical Risk Minimization (ERM) oracle for the function class $\cF$; an ERM oracle, formally defined in \Cref{app:smoothedonline} (see Definition \ref{def:ermoracle}), returns the minimizer of a weighted, cumulative loss function defined over the function class $\cF$.  Once again, for the sake of simplicity, we state our results for the case of binary valued $\cF$ and defer the more general statement and proof to the appendix.
\begin{theorem}\label{thm:binaryimproper}
    Suppose that $\cF: \cX \to \left\{ \pm 1 \right\}$ is a function class with VC dimension $\vc{\cF}$ and that $\ell: [-1,1] \times [-1,1] \to [0,1]$ is a loss function that is convex and 1-Lipschitz in the first argument.  Then there is an \emph{improper} algorithm requiring 2 calls to the ERM oracle per time $t$ such that if the adversary is $(f, \sigma)$-smoothed, then the regret is bounded as follows:
    \begin{align}\label{eq:binaryimproperbound}
        \ee\left[ \reg_T \right] \lesssim \inf_{\alpha > 0}\left\{ \alpha T + \sqrt{\vc{\cF} \cdot T \cdot \log(T) \cdot (f')^{-1}\left( \frac{1}{\alpha \sigma} \right)} \right\}.
    \end{align}
\end{theorem}
We prove \Cref{thm:binaryimproper} in \Cref{app:smoothedonline}, where we apply Lemma \ref{lem:coupling} to the argument of \citet{block2022smoothed}.  We instantiate the bound in \eqref{eq:binaryimproperbound} in two cases, Renyi divergence (\Cref{ex:renyi}) and KL Divergence (\Cref{ex:kl}).  If we assume that our adversary is smoothed in the sense of Renyi divergence, then optimizing $\alpha$ leads us to an oracle-efficient algorithm attaining regret scaling like $\widetilde{O}\left( \vc{\cF}^{\frac{\lambda - 1}{2 \lambda - 1}} \cdot T^{\frac{\lambda}{2 \lambda - 1}} \cdot \sigma^{- \frac{1}{2\lambda - 1}} \right)$.  Noting that if $\norm{\frac{d p_t}{d\mu}}_\infty \leq (\sigma')^{-1}$ then we may take $\sigma = (\sigma')^{\lambda - 1}$, we observe that in the limit as $\lambda \uparrow \infty$, we recover the $\widetilde{O}\left( \sqrt{\vc{\cF} \cdot T / \sigma'} \right)$ rate from \citet[Theorem 7]{block2022smoothed}.  In the special case where $\lambda  =2$, we see that the regret scales like $\widetilde{O}\left( (\vc{\cF}/\sigma)^{1/3} \cdot T^{2/3}\right)$.  On the other hand, if we make the weaker assumption that the adversary is only smoothed in the KL sense, then \Cref{thm:binaryimproper} only recovers a regret that scales as $\widetilde{O}\left( \log(d) T / (\sigma \log(T)) \right)$, which is sublinear in $T$ but very slow.  

We turn now to the case of proper algorithms.  As in \citet{block2022smoothed}, we instantiate Follow the Perturbed Leader (FTPL) with a perturbation by a Gaussian process; again, we apply our Lemma \ref{lem:coupling} to the proof techniques found in \citet[Appendix E]{block2022smoothed}.  For the sake of simplicity, we restrict our focus to binary valued function classes with linear loss.
\begin{theorem}\label{thm:properinformal}
    Suppose that we are in the situation of \Cref{thm:binaryimproper}, with the loss function $\ell$ being linear, i.e., $\ell(\yhat, y) = (1 - \yhat \cdot y) / 2$.  Suppose further that our adversary is $(f,\sigma)$-smooth in the sense of Renyi Divergence, i.e., for some $\lambda \geq 2$, $\dren{p_t}{\mu} \leq 1/\sigma$ for all $p_t$.  Then there is a \emph{proper} algorithm requiring only 1 call to the ERM oracle per round such that the regret is bounded as follows:
    \begin{align}
        \ee\left[ \reg_T \right] = \widetilde{O}\left(\sqrt{\vc{\cF}} \cdot T^{\frac{2 \lambda + 1}{4 \lambda - 1}} \cdot \sigma^{-\frac{1}{4\lambda - 1}} \right).
    \end{align}
\end{theorem}
Note that our regret in \Cref{thm:properinformal} actually improves on that of \cite[Theorem 10]{block2022smoothed} in the case where we take $\lambda \uparrow \infty$.  Indeed, if we are in the strongly smooth regime such that the Radon-Nikodym derivative of the adversary's distribution is uniformly bounded by $\sigma'^{-1}$, then in the limit we recover an expected regret scaling like $\widetilde{O}\left( \sqrt{\vc{\cF} \cdot T} \cdot (\sigma')^{- \frac 14}  \right)$, which matches that of the instantiation of FTPL found in \cite{haghtalab2022oracle} for discrete $\cX$.  Thus, by examining $f$-smoothed adversaries, we anwer an open question of \citet{block2022smoothed} on improving the dependence on $\sigma'$ of the expected regret of FTPL with a Gaussian perturbation.

We leave as an interesting further direction the question regarding the tightness of the regret of the algorithms in \Cref{thm:binaryimproper,thm:properinformal}.  As shown in \citet{block2022smoothed,haghtalab2022oracle}, even in the case of strongly smoothed adversaries, there is a statistical-computational gap wherein the dependence of the expected regret for an oracle-efficient algorithm on $\sigma$ must be polynomial, but \Cref{thm:minimax} yields a statistical rate that is polylogarithmic in the same.  Even in the adversarial setting, however, it is unknown if such an exponential gap exists for oracle-efficient \emph{improper} algorithms \citep{hazan2016computational}.

Finally, we observe that \Cref{thm:properinformal} only applies to $f$-smoothed adversaries in the Renyi sense for $\lambda \geq 2$.  Our proof proceeds by a change of measure argument, wherein we replace an expectation over the base measure $\mu$ by an expectation over the adversary's distribution $p_t$; for a weaker $f$-divergence like KL, the analogous statement would require bounding an exponential moment, which would require significantly stronger analysis.  We leave the question of existence of oracle-efficient proper algorithms for KL smoothed adversaries as yet another interesting further direction.

\section*{Acknowledgements}
AB acknowledges support from the National Science Foundation Graduate Research Fellowship under Grant No. 1122374 as well as support from ONR under grant N00014-20-1-2336 and DOE under grant DE-SC0022199. YP was supported in part by the MIT-IBM Watson AI Lab and by the NSF grant
CCF-2131115.  We also acknowledge an anonymous reviewer for pointing us to a relevant reference as well as Gergely Flamich and William Lennie Wells for pointing out an error in the proof of \Cref{lem:expectedrestrictedrdbound} in an earlier version of this paper.

\bibliographystyle{plainnat}
\bibliography{references}

\appendix

\section{Related Work}\label{app:relatedwork}

\paragraph{Rejection Sampling}
As mentioned in the introduction, rejection sampling was pioneered by \citet{von195113} and has received much attention due to its simplicity and general application, with far too many references to list \citep{gilks1992adaptive,liu1996metropolized,flury1990acceptance,harviainen2021approximating,bauer2019resampled}.  More recently, the information theory community has been interested in improving the bounds on the sample complexity of \emph{exact} rejection sampling under assumed additional structure, such as a Renyi divergence bound \citep{liu2018rejection,harsha2007communication}.  

Perhaps surprisingly, \emph{approximate} rejection sampling has received relatively little attention.  Several works have proposed a tradeoff between the sample complexity of approximate rejection sampling and the accuracy of the produced sample; one example is \citet{grover2018variational}, which demonstrated a qualitative monotonicity property that describes this tradeoff in a particular family of distributions, without providing any quantitative guarantees.  Some works in cryptography \citep{lyubashevsky2012lattice,zheng2021rejection,agrawal22} have provided upper bounds on the sample complexity in the particular case of the discrete Gaussian family on a lattice.  Even more recently, a concurrent work \citep{devevey2022rejection} demonstrates upper and lower bounds for \emph{exact} rejection sampling as well as some upper bounds for \emph{approximate} rejection sampling for probability distributions on the discrete hypercube.  In contradistinction to these previous works, our results hold for arbitrary probability distributions and significantly more general $f$-divergences.  In addition \citet{eikema2022approximate} examines the approximate rejection sampling problem and proves several qualitative results about the tradeoff in the sample complexity of approximate rejection sampling and the accuracy of the produced sample.  Our work provides the first sharp quantitative guarantees for approximate rejection sampling in the general case.

\paragraph{Smoothed Online Learning}
The study of online learning dates back decades with too many references to list here.  A good introduction to the general field can be found in \citet{cesa2006prediction}.  More recently, there has been a surge of interest in sequential analogues of statistical learning phenomena \citep{rakhlin2015sequential,rakhlin2012relax,block2021majorizing}.  Due to the statistical and computational challenges of this regime, however, several works have proposed the \emph{smoothed} online learning setting \citep{rakhlin2011online,haghtalab2020smoothed}, with \citet{haghtalab2022smoothed,block2022smoothed} providing statistical rates defining the difficulty of a smoothed online learning problem and \citet{block2022smoothed,haghtalab2022oracle,block2022efficient} providing oracle-efficient algorithms.  In this work, we generalize the results of \citet{block2022smoothed,haghtalab2022oracle} to what we call the $f$-smoothed online learning setting, where the adversary is constrained to only be smooth in a weaker sense.

\paragraph{Out-of-Distribution Learning}
Importance sampling was introduced in \citet{kloek1978bayesian} and studied extensively thereafter due to its wide applicability.  Again, there are far too many references in this popular field to include here, but a few standard treatments are \citet{liu2001monte,srinivasan2002importance,tokdar2010importance}.  Most similar to our work are \citet{chatterjee2018sample,cortes2010learning}.  In the first work, the authors precisely compute the sample complexity of importance sampling to estimate the mean of a given function $f$ under some target measure $\nu$.  They observe that $\widetilde{\Theta}\left( e^{\dkl{\nu}{\mu}} \right)$ samples are both necessary and sufficient to do this and provide several instantiations of their main bound.  Unfortunately, their bounds are too weak to apply to the problem of estimating means uniformly over a large function class, as is required for learning theory.  In \citet{cortes2010learning}, the authors prove upper bounds on the sample complexity of importance sampling assuming that the function class $\cF$ has finite pseudo-dimension, under both bounded likelihood and bounded Renyi constraints.  They also prove a lower bound.

\section{Comparison to Importance Sampling for Uniform Mean Estimation}\label{sec:importance}

In this appendix, we apply our main result to uniform mean estimation.  More specifically, suppose that the learner has access to $X_1, \dots, X_n \sim \mu$ independent samples and, for some other measure $\nu$, wishes to estimate $\ee_{Y \sim \nu}\left[ f(Y) \right]$ for all functions $f$ in some class $\cF$.  A natural quesiton is how large $n$ must be and how close $\nu$ has to be to $\mu$ in order for our estimates to be within $\epsilon$ of $\ee\left[ f(Y) \right]$ with high probability.  In this section, we will compare two common approaches.  The first uses importance sampling, where we take our estimate to be
\begin{align}
    I_n(f) = \frac 1n \sum_{i  =1}^n \frac{d\nu}{d\mu}(X_i) f(X_i),
\end{align}
while the second is to use rejection sampling to generate $n'$ independent samples from $\nu$ and then simply use the empirical mean of these samples to estimate $\ee_{Y \sim \nu}\left[ f(Y) \right]$.  We begin by considering importance sampling.

In the special case where $\cF = \{f\}$ is a singleton, the following theorem of \citet{chatterjee2018sample} fully answers the question of sample complexity:
\begin{theorem}[Theorem 1.1, \citet{chatterjee2018sample}]\label{thm:chaterjee}
    Suppose that $X_1, \dots, X_n \sim \mu$ are independent and suppose that $n \geq \exp\left( \dkl{\nu}{\mu} \right)$.  If $f$ is a real-valued, measurable function, then it holds that
    \begin{align}
        \ee\left[ \abs{I_n(f) - \ee_\nu\left[ f(Y) \right]} \right] \leq \norm{f}_{L^2(\nu)}\left( \left( \frac{e^{\dkl{\nu}{\mu}}}{n} \right)^{\frac{1}{4}}  + 2 \sqrt{\pp_{Y \sim \nu}\left( \frac{d\nu}{d\mu}(Y) > \sqrt{e^{\dkl{\nu}{\mu}} \cdot n} \right)}\right).
    \end{align}
    Moreover, if $n \leq e^{\dkl{\nu}{\mu}}$ then with probability at least $1 - \delta$,
    \begin{align}
        I_n(1) \leq \sqrt{\frac{n}{e^{\dkl{\nu}{\mu}}}} + \frac{\pp\left( \frac{d\nu}{d\mu}(Y) \leq \sqrt{e^{\dkl{\nu}{\mu}} \cdot n} \right)}{1 - \delta}.
    \end{align}
\end{theorem}
Thus \Cref{thm:chaterjee} shows that $\Theta\left( \exp(\dkl{\nu}{\mu}) \right)$ samples are necessary and sufficient in order for importance sampling to generate an estimate that is close in expectation to the true mean.  Unfortunately, the above guarantee is too weak to be applied to large function classes.  While \citet{cortes2010learning} provides several bounds on importance sampling that hold uniformly in a function class with bounded pseudo-dimension, we provide here a generalization of their result that holds for most Donsker function classes.  Before doing this, we need to define the relevant notion of complexity of the function class: that of the bracketing number.  For more details, see \citet{gine2021mathematical}.
\begin{definition}
    Let $\cF: \cX \to [-1,1]$ be a real-valued function class.  Given real-valued functions $f_-, f_+$, we define the bracket $[f_-, f_+]$ as the set of functions $f \in \cF$ such that $f_- \leq f \leq f_+$ pointwise.  We say that $[f_-, f_+]$ is an $\epsilon$ bracket with respect to $\mu$ if $\norm{f_+ - f_-}_{L^2(\mu)} \leq \epsilon$.  We define the bracketing number $\nbrack{\cF}{\epsilon}$ as the minimal number of $\epsilon$-brackets such that $\cF$ is contained in the union of these brackets.  Finally, we say that $\cF$ has finite bracketing integral if
    \begin{align}
        \int_0^2 \sqrt{\log \nbrack{\cF}{\epsilon}} d \epsilon < \infty.
    \end{align}
\end{definition}
While there are more general complexity assumptions under which our conclusions hold, for the sake of simplicity, we consider the bracketing integral.  We have the following result:
\begin{theorem}\label{thm:importancesamplingupper}
    Suppose that $\cF: \cX \to [-1,1]$ is a real-valued function class with finite bracketing integral with respect to $\mu$.  Suppose further that $\chisq{\nu}{\mu} < \infty$.  Then the following inequality holds:
    \begin{align}\label{eq:importanceupper}
        \ee\left[ \sup_{f \in \cF} \abs{I_n(f) - \ee_\nu\left[ f(Y) \right]} \right] \lesssim \sqrt{1 + \chisq{\nu}{\mu}} \cdot \max\left( n^{-1/3}, \sqrt{\frac{1 + \chisq{\nu}{\mu}}{n}} \cdot \int_0^2 \sqrt{\log N_{[]}\left( \cF, \alpha \right)} d \alpha \right).
    \end{align}
\end{theorem}
The proof of \Cref{thm:importancesamplingupper} rests on the following result.  For future reference, we denote by $P_n$ the empirical measure on $n$ independently sampled points from $\mu$.
\begin{theorem}[Theorem 1.1 from \citet{bercu2002concentration}]\label{thm:bercu}
    Suppose that $\cF$ is a real valued function class with finite bracket numbers satisfying
    \begin{align}
        E_\cF = \sup_{n > 0} \ee\left[\sup_{f \in \cF} \max\left( P_n\left( f - \ee_\mu[f] \right), 0 \right)  \right] < \infty.
    \end{align}
    For any $\delta > 0$ and $\alpha > \sqrt{2}$, there exist constants $\theta, n_0$ depending on $\delta$ and $\alpha$ such that
    \begin{align}
        \pp\left( \sup_{f \in \cF} \frac{P_n(f - \ee_\mu\left[ f \right])}{\sqrt{P_n\left( f - \ee_\mu[f] \right)^2}} \geq \frac{x + E_\cF}{\sqrt{n}} \right) \leq e^{- \frac{x^2}{4 \alpha^2 (1 + \delta)}}
    \end{align}
    for all $n \geq n_0$ and $x \leq \theta \sqrt{n}$.
\end{theorem}
We observe that \citet[Theorem 3.5.13]{gine2021mathematical} tells us that, up to a constant, $E$ is bounded by the bracketing integral of $\cF$.  We are now ready to prove our importance sampling bound:
\begin{proof}[Proof of \Cref{thm:importancesamplingupper}]
    We will apply \Cref{thm:bercu} to the function class
    \begin{align}
        \widetilde{\cF} = \frac{d\nu}{d\mu} \cdot \cF = \left\{ x \mapsto \frac{d\nu}{d\mu}(x) \cdot f(x) | f \in \cF \right\}.
    \end{align}
    We first note that Cauchy-Schwarz tells us that if $[f_-, f_+]$ is an $\epsilon$-bracket for $\mu$ then $\left[ \frac{d\nu}{d\mu} \cdot f_-, \frac{d\nu}{d\mu} \cdot f_+ \right]$ is a $(\sqrt{1+\chisq{\nu}{\mu}} \cdot \epsilon)$-bracket for $\mu$.  We further note that
    \begin{align}
        I_n(f) = P_n\left( \frac{d \nu}{d\mu} \cdot f\right) &&  \ee_\nu\left[ f(Y) \right] \ee_\mu\left[ \frac{d \nu}{d\mu} \cdot f \right].
    \end{align}
    Now note that it suffices to prove an upper tail bound and symmetry will imply a lower tail bound as well.  We will apply \Cref{thm:bercu}, but first we must bound the relevant quantities.  Observe that \citet[Theorem 3.5.13]{gine2021mathematical} implies that $E_\cF$ is bounded by the bracketing integral of $\cF$; combining this with our observation on the relationship between $\epsilon$-brackets of $\cF$ and those of $\widetilde{\cF}$, we see that
    \begin{align}
        E_{\widetilde{\cF}} \lesssim \sqrt{1 + \chisq{\nu}{\mu}} \cdot \int_0^2 \sqrt{\log N_{[]}\left( \cF, \alpha \right) d \alpha}.
    \end{align}
    We further observe by Cauchy-Schwartz that
    \begin{align}
        \sqrt{P_n\left( f - \ee_\mu[f] \right)^2} &\leq 2 \cdot \sqrt{P_n\left( \frac{d\nu}{d\mu} \cdot f \right)^2 + \ee_\mu\left[ \left( \frac{d\nu}{d\mu} \cdot f \right)^2 \right]} \\
        &\leq 2 \cdot \sqrt{P_n\left( \frac{d\nu}{d\mu} \right)^2 \cdot P_n(f)^2 + (1 + \chisq{\nu}{\mu}) \cdot \ee_\mu\left[ f^2 \right] } \\
        &\leq 2 \cdot \sqrt{P_n\left( \frac{d\nu}{d\mu}^2 \right) + 1 + \chisq{\nu}{\mu}},
    \end{align}
    where we used the fact that $\cF$ is uniformly bounded.  We further note by Markov's inequality that
    \begin{align}\label{eq:markovempirical}
        \pp\left( P_n\left( \frac{d\nu}{d\mu} \right)^2 > u^2 \right) \leq \frac{1 + \chisq{\nu}{\mu}}{u^2}.
    \end{align}
    By \Cref{thm:bercu}, it holds that
    \begin{align}
        \pp\left( \sup_{f \in \cF} I_n(f) - \ee_\nu\left[ f \right] > t \text{ and } P_n\left( \frac{d\nu}{d\mu} \right)^2 \leq u^2 \right) \leq \exp\left(- C \left( \sqrt{\frac{n}{1 + u^2 + \chisq{\nu}{\mu}}} \cdot t - E_{\widetilde{\cF}} \right)^2 \right).
    \end{align}
    Now, setting
    \begin{align}
        u = n^{\frac{1}{6}} \cdot t^{\frac 23} \cdot \sqrt{1 + \chisq{\nu}{\mu}}
    \end{align}
    we see that as long as
    \begin{align}
        t \geq C \sqrt{\frac{1 + \chisq{\nu}{\mu}}{n}} \cdot E_{\widetilde{\cF}},
    \end{align}
    it holds that
    \begin{align}
        \pp\left( \sup_{f \in \cF} I_n(f) - \ee_\nu\left[ f \right] > t \text{ and } P_n\left( \frac{d\nu}{d\mu} \right)^2 \leq u^2 \right) \leq C \exp\left(- C \frac{n^{2/3} t^{2/3}}{(1 + \chisq{\nu}{\mu})} \right).
    \end{align}
    Applying \eqref{eq:markovempirical}, we see that
    \begin{align}
        \pp\left( \sup_{f \in \cF} I_n(f) - \ee_\nu\left[ f \right] > t\right) &\leq  \frac{1 }{n^{1/3} \cdot t^{4/3}} + C \exp\left(- C \frac{n^{2/3} t^{2/3}}{(1 + \chisq{\nu}{\mu})} \right).
    \end{align}
    The same result holds for the lower tail by applying the identical argument to $-\cF$.  To conclude, we observe that
    \begin{align}
        \ee\left[ \sup_{f \in \cF} \abs{I_n(f) - \ee_\nu\left[ f \right]} \right] &= \int_0^\infty \pp\left(  \sup_{f \in \cF} \abs{I_n(f) - \ee_\nu\left[ f \right]} > t \right) d t \\
        &\lesssim \sqrt{\frac{1 + \chisq{\nu}{\mu}}{n}} \cdot E_{\widetilde{\cF}} + \frac{\sqrt{1 + \chisq{\nu}{\mu}}}{n^{\frac 13}} \\
        &\lesssim \sqrt{1 + \chisq{\nu}{\mu}} \cdot \max\left( n^{-1/3}, \sqrt{\frac{1 + \chisq{\nu}{\mu}}{n}} \cdot \int_0^2 \sqrt{\log N_{[]}\left( \cF, \alpha \right)} d \alpha \right)
    \end{align}
    as desired.
\end{proof}
Note that as $n\uparrow \infty$, the rates given by \Cref{thm:importancesamplingupper} scale like $O\left(n^{-1/3}\right)$.  One improvement on these rates follows from the work of \citet{cortes2010learning}, where they assume that the function class $\cF$ has finite pseudo-dimension and obtain rates that scale like $O\left(  n^{-3/8}\right)$; note that the property of a class having finite bracketing integral is distinct from that of having finite pseudo-dimension; indeed \citet[Proposition 1.7]{van2013universal} shows that bracketing numbers can be arbitrarily large even for classes of finite VC dimension, whereas \cite{gine2021mathematical} shows that classes with infinite VC dimension such as Sobolev spaces still may have small bracketing numbers.  In \citet[Proposition 2]{cortes2010learning}, the authors show a lower bound for importance sampling assuming that $\frac{d\nu}{d\mu}$ is uniformly bounded and the sample size is large relative to this uniform bound.  In essence, this shows that if $\chisq{\nu}{\mu}$ is infinite, then we cannot hope to get an importance sampling estimator that converges to the population mean at a $\Theta\left( n^{- \frac 12} \right)$ rate.  We leave as an interesting direction for future research the problem of closing the gap between these two rates.

We now turn to our second estimator: using rejection sampling to produce independent samples from $\nu$ and taking the sample mean.  More precisely, given $X_1, \dots, X_n \sim \mu$ independent and for some $m \in \bbN$ dividing $n$, suppose we partition $[n]$ into sets of size $m$ and conduct the rejection sampling procedure of \Cref{thm:upperbound} on each subset, generation $n / m$ independent samples $X_1', \dots, X_{n/m}' \sim \nu_M$ independent with probability at least $1 - n \delta$ for some $\delta$.  For given $m,n$, let
\begin{align}\label{eq:rejsamplingmeanestimation}
    J_{m,n}(f) = \frac m n \cdot \sum_{i  =1}^{n/m} f(X_i')
\end{align}
denote the rejection sampling estimate of $\ee_\nu\left[ f(Y) \right]$.  We have the following result:
\begin{corollary}\label{cor:rejsamplingmeanestimation}
    Suppose that $\cF: \cX \to [-1,1]$ is a real-valued function class with finite bracketing integral with respect to $\mu$ and that $\divf{\nu}{\mu} < \infty$.  Suppose further that for some $\frac 12 > \epsilon > 0$,
    \begin{align}
        m \geq 4 \log\left( \frac{1}{\epsilon} \right) \cdot (f')^{-1}\left( \frac{\divf{\nu}{\mu}}{\epsilon} \right).
    \end{align}
    Then it holds that
    \begin{align}
        \ee\left[ \sup_{g \in \cF} \abs{J_{m,n}(g) - \ee_\nu\left[ g(Y) \right]} \right] \lesssim \left( \int_0^2 \sqrt{\log \nbrack{\cF}{\alpha}} d \alpha \right) \cdot \sqrt{\frac{m}{n}} + 2 \epsilon
    \end{align}
\end{corollary}
\begin{proof}
    We begin by invoking \Cref{thm:upperbound} and observing that we may choose $X_1', \dots, X_{n/m}' \sim \nutil$ independent for some $\nutil$ satisfying $\tv\left( \nutil, \nu \right) \leq \epsilon$.  Observe that by \citet[Theorem 3.5.13]{gine2021mathematical}, it holds that
    \begin{align}
        \ee\left[ \sup_{g \in \cF} \abs{J_{m,n}(g) - \ee_{\nutil}\left[ g(\widetilde{Y}) \right]} \right] \lesssim \sqrt{\frac{m}{n}} \cdot \left( \int_0^2 \sqrt{\log \nbrack{\cF}{\alpha}} d \alpha \right).
    \end{align}
    Now, noting that $g$ takes values in $[-1,1]$, we see that
    \begin{align}
        \ee\left[ \sup_{g \in \cF} \abs{\ee_{\nutil}\left[ g(\widetilde{Y}) \right] - \ee_\nu\left[ g(Y) \right]} \right] \leq 2 \cdot \tv\left(\nutil, \nu  \right) \leq 2 \epsilon.
    \end{align}
    The result follows.
\end{proof}
Rescaling, we see that \Cref{cor:rejsamplingmeanestimation} tells us that if we wish to apply the estimator \eqref{eq:rejsamplingmeanestimation}, we need
\begin{align}
    n = \Theta\left( \frac{\log\left( \frac 1\epsilon \right) \cdot (f')^{-1}\left( \frac{\divf{\nu}{\mu}}{\epsilon} \right)}{\epsilon^2} \cdot \left( \int_0^2 \sqrt{\log \nbrack{\cF}{\alpha}} d \alpha \right)^2\right)
\end{align}
samples in order for our estimate to be within $\epsilon$ of the true population mean, uniformly in the function class $\cF$. In the special case where our $f$-divergence is $\chisq{\nu}{\mu}$, we see that $\widetilde{\Theta}\left(  \chisq{\nu}{\mu} \cdot \epsilon^{-3} \right)$ samples suffice to recover a uniform estimate of the mean within $\epsilon$ error.  Note that this matches the rate given by importance sampling: according to \Cref{thm:importancesamplingupper}, $O\left( \chisq{\nu}{\mu} \cdot \epsilon^{-3} \right)$ samples suffice to recover $\epsilon$-accurate uniform estimates in expectation in the same situation.  On the other hand, \Cref{cor:rejsamplingmeanestimation} applies to arbitrary $f$-divergences and thus is significantly more general; we defer to future work the question of when \Cref{thm:importancesamplingupper} can be extended to more general $f$-divergences.  We observe, however, that the analysis of the rejection-sampling estimator is essentially tight while that of the importance sampling estimator is potentially loose.  One case where importance sampling improves on rejection sampling is when $\cF$ has finite pseudo-dimension; in this case, the results of \citet[Theorem 3]{cortes2010learning} tell us that $\widetilde{O}\left( \sqrt{\chisq{\nu}{\mu}} \cdot \epsilon^{-8/3} \right)$ samples suffice for importance sampling, which is strictly better than the rate for rejection sampling derived above.

\section{Sequential Complexities and Minimax Regret}\label{app:seqcomplexities}
In this section, we recall some basic definitions of different notions of complexity of an online learning problem and how they relate to the regret.  These results will be used throughout \Cref{app:smoothedonline}.  Many of the adversarial complexities were introduced in \citet{rakhlin2015sequential} and we closely follow the presentation in that work as well as that in \citet{block2022smoothed}.  We begin by recalling the definition of scale-sensitive VC dimension from \citet{bartlett1994fat}:
\begin{definition}\label{def:vcdimension}
    Let $\cF : \cX \to \rr$ be a function class.  For any $\alpha > 0$ and points $X = \{x_1, \dots, x_m \} \subset \cX$, we say that the set of points $X$ shatters $\cF$ at scale $\alpha$ with witnesses $s_1, \dots, s_m \in \rr$ if for all $m$-tuples of signs $\epsilon = (\epsilon_1, \dots, \epsilon_m)$, there exists some $f_\epsilon \in \cF$ such that for all $1 \leq i \leq m$:
    \begin{align}
        \epsilon_i (f_\epsilon(x_i) - s_i) \geq \frac \alpha 2.
    \end{align}
    We let $\vc{\cF, \alpha}$ denote the maximal $m$ such that there exists a set $X$ of size $m$ shattering $\cF$ and let $\vc{\cF} = \sup_\alpha \vc{\cF, \alpha}$.
\end{definition}
It is well known from \citet{kearns1994efficient,bartlett1994fat} that $\vc{\cF, \alpha}$ charaterize the learnability of $\cF$ when the data are independent and identically distributed (i.e., the batch setting).  Another related quantity is the Rademacher complexity, defined for some set $x_1, \dots, x_n \in \cX$ to be
\begin{align}\label{eq:rademachercomplexity}
    \rad_n(\cF) = \ee_\epsilon\left[ \sup_{f\in \cF} \sum_{i = 1}^n \epsilon_i f(x_i) \right],
\end{align}
where the $\epsilon_i$ are independent Rademacher random variables.  The following result can be found in \citet{rudelson2006combinatorics}:
\begin{proposition}\label{prop:rudelsonvershynin}
    Let $\cF: \cX \to [-1,1]$ be a function class.  Then it holds that
    \begin{align}
        \rad_n(\cF) \lesssim \inf_{\gamma > 0} \gamma n + \sqrt{n} \int_\gamma^1 \sqrt{\vc{\cF, \delta}}  d \delta.
    \end{align}
\end{proposition}
A similar result applying to the fully adversarial setting was proved in \citet{block2021majorizing}.  In order to state the analogue precisely, we first recall the notion of distribution-dependent sequential Rademacher complexity from \citet{rakhlin2011online}.  For a given depth $T$, full binary tree $\bx$, with vertices labelled by elements of some space $\cX$ and a path $\epsilon \in \left\{ \pm 1 \right\}^T$, we denote by $\bx_t(\epsilon)$ the vertex at step $t$ of the path given by $\epsilon$ starting at the root that takes the right child at time $s$ if $\epsilon_s = 1$ and the left child otherwise.  For a given adversary producing $x_1, \dots, x_T$, let $P_{x_{1:T}}$ denote the join distribution of $x_1, \dots, x_T$ and let $p_t$ denote the distribution of $x_t$ conditional on the history.  Define the measure $\rho_{P_{x_{1:T}}}$ on an ordered pair $(\bx, \bx')$ of depth $T$ binary trees with labels in $\cX$ recursively as follows.  First sample $\bx_0, \bx_0' \sim p_0$ indpeendently.  Now suppose that $t > 0$ anda for any $s < t$, let
\begin{align}
    \chi_s(\epsilon) = \begin{cases}
        \bx_s(\epsilon) & \epsilon_s = 1 \\
        \bx_s'(\epsilon) & \epsilon_s = -1
    \end{cases}.
\end{align}
Sample $\bx_t(\epsilon), \bx_t'(\epsilon) \sim p_t(\cdot | \chi_1(\epsilon), \dots, \chi_{t-1}(\epsilon))$ independently and proceed until two depth $T$ binary trees are constructed.  We now define the distribution-dependent sequential Rademacher complexity:
\begin{definition}[Definition 2 from \citet{rakhlin2011online}]
    Let $\cF: \cX \to [-1,1]$ be a function class and fix a distribution $P_{x_{1:T}}$ on tuples $(x_1, \dots, x_T)$.  We define the distribution-dependent sequential Rademacher complexity as follows:
    \begin{align}
        \radseq_T\left(\cF, P_{x_{1:T}}\right) = \ee_{(\bx, \bx') \sim \rho_{P_{x_{1:T}}}}\left[ \ee_\epsilon\left[ \sup_{f \in \cF} \sum_{t = 1}^T \epsilon_t f(\bx_t(\epsilon)) \right] \right].
    \end{align}
    For a given class of distributions $\scrD = \left\{ P_{x_{1:T}} \right\}$, define
    \begin{align}
        \radseq_T\left( \cF, \scrD \right) = \sup_{P_{x_{1:T}} \in \scrD} \radseq_T \left( \cF, P_{x_{1:T}} \right).
    \end{align}
\end{definition}
Note that if $\scrD$ is the set of all distributions on $\cX$ then the notion of standard, sequential Rademacher complexity from \citet{rakhlin2015sequential} is recovered.  The reason for introducing this admittedly technical notion of complexity is the following result:
\begin{theorem}[Theorem 3 and Lemma 20 from \citet{rakhlin2011online}]\label{thm:seqradcomplexity}
    Let $\cF: \cX \to [-1,1]$ be a function class and $\ell: [-1,1] \times [-1,1] \to [0,1]$ a loss function Lipschitz in the first argument.  Suppose that we are in the online learning setting described in \Cref{sec:smoothedonline} and the adversary is constrained to choose $x_{1:T}$ according to some distribution in $\scrD$.  Then there exists an algorithm such that
    \begin{align}
        \ee\left[ \reg_T \right] \lesssim \radseq_T(\cF, \scrD).
    \end{align}
\end{theorem}
Returning to combinatorial notions of complexity, \citet{rakhlin2015sequential} introduced the following sequential analogue of the scale-sensitive VC dimension:
\begin{definition}
    Let $\cF : \cX \to \rr$ be a function class.  For any $\alpha > 0$ and complete binary tree $\bx$ of depth $m$, we say that $\cF$ is shattered by $\bx$ with witness (complete binary) tree $\bs \in \rr^{\abs{\bx}}$ if for all $m$-tuples of signs $\epsilon = (\epsilon_1, \dots, \epsilon_n)$, there exists some $f_\epsilon \in \cF$ such that for all $1 \leq i \leq m$, it holds that
    \begin{align}
        \epsilon_i \left( f_\epsilon(\bx_i(\epsilon)) - \bs_i(\epsilon) \right) \geq \frac \alpha 2.
    \end{align}
    We let $\fat{\cF, \alpha}$, the sequential fat-shattering dimension, denote the maximal $m$ such that there exists a tree $\bx$ of depth $m$ that shatters $\cF$.
\end{definition}
Note that in general $\fat{\cF, \alpha} \gg \vc{\cF, \alpha}$ and the difference can be infinite, as is the case for thresholds on the unit interval for example.  For finite domains $\cX$, however, a reverse bound is possible.  We make use of the following result:
\begin{lemma}[Lemma 21 from \citet{block2022smoothed}]\label{lem:fatvcbound}
    Let $\cF: \cX \to [-1,1]$ be a function class and let $\fat{\cF, \alpha}$ denote the sequential fat-shattering dimension at scale $\alpha$.  Then there exist universal constants $C,c$ such that for any $\beta > 0$, the following inequality holds:
    \begin{align}
        \fat{\cF, \alpha} \leq C \cdot \vc{\cF, c \beta \alpha} \cdot \log^{1 + \beta}\left( \frac{C \abs{\cX}}{\vc{\cF, c \alpha} \alpha} \right).
    \end{align}
\end{lemma}
In order to relate these notions of complexity back to the problem at hand, we make use of the following result:
\begin{proposition}[Corollary 18 from \citet{block2021majorizing}]\label{prop:adversarialbound}
    Let $\cF: \cX \to [-1,1]$ be a function class.  For any distribution $P_{x_{1:T}}$, the following bound holds:
    \begin{align}
        \radseq_T\left( \cF, P_{x_{1:T}} \right) \lesssim \inf_{\gamma > 0} \gamma T + \sqrt{T} \int_\gamma^1 \sqrt{\fat{\cF, \delta}}  d \delta.
    \end{align}
\end{proposition}
Combining Proposition \ref{prop:adversarialbound} and Lemma \ref{lem:fatvcbound} implies the following result from \citet{block2022smoothed}:
\begin{lemma}\label{lem:minimaxregretfinitedomain}
    Suppose that we are in the situation of Proposition \ref{prop:adversarialbound} and, furthermore, $\abs{\cX} < \infty$.  Then for any $P_{x_{1:T}}$, the following holds:
    \begin{align}
        \radseq_T\left(\cF, P_{x_{1:T}}\right) \lesssim \inf_{\substack{\beta, \gamma > 0}} \gamma T + \sqrt{T \cdot \log^{1 + \beta}(\abs{\cX})} \int_\gamma^1 \sqrt{\vc{\cF, c \beta \delta} \cdot \log^{1+\beta}\left( \frac{1}{\vc{\cF, c \delta} \delta} \right)}d \delta.
    \end{align}
\end{lemma}

\section{Proofs from Section \ref{sec:rejsampling}}\label{app:rejsampling}
\subsection{Upper Bounds}
In this section, we prove the main upper bound, \Cref{thm:upperbound}.  We begin by stating and proving the standard gaurantees for rejection sampling for the sake of completeness.
\begin{lemma}\label{lem:rejsampling}
    Let $\mu, \nu$ be two measures on $(\cX, \scrF)$ and suppose that $X \sim \mu$.  Suppose that $\mu, \nu$ satisfy the condition that for some $M < \infty$, $\norm{\frac{d \nu}{d\mu}}_\infty \leq M$.  Let $\xi$ be a binary-valued random variable such that, conditional on $X$, the probability that $\xi = 1$ is given by $\frac 1M \frac{d \nu}{d \mu}(X)$.  Then $\pp(\xi = 1) = \frac {1}M$ and for any $A \in \scrF$, it holds that
    \begin{align}
        \pp\left( X \in A | \xi = 1 \right) = \nu(A).
    \end{align}
    Thus, if $X_1, \dots, X_n$ are sampled independently from $\mu$ and $\xi_1, \dots, \xi_n$ are constructed as above, with probability at least $1 - e^{- \frac{n}{M}}$ at least one of the $\xi_j$ is equal to 1.
\end{lemma}
\begin{proof}
    To prove the first statement, by the tower property of conditional expectation, we see that
    \begin{align}
        \pp\left( \xi = 1 \right) = \ee\left[ \frac{1}{M} \frac{d \nu}{d\mu}(X)  \right] = \frac {1}M
    \end{align}
    by the definition of the Radon-Nikodym derivative.  To prove the second statement, we see that
    \begin{align}
        \pp\left( X \in A | \xi = 1 \right) &= \frac{\pp\left( X \in A \text{ and } \xi = 1 \right)}{\pp(\xi = 1)} \\
        &= M \cdot \pp\left( X \in A \text{ and } \xi = 1 \right) \\
        &= M \cdot \ee\left[ \I[X \in A] \frac 1M \frac{d \nu}{d \mu}(X) \right] \\
        &= \nu(A)
    \end{align}
    as desired.  Finally, the last statement follows because
    \begin{align}
        \pp\left( \xi_j = 1 \text{ for some } j \right) &= 1 - \pp\left(\xi_j = 0 \text{ for all } j  \right) \\
        &= 1 - \prod_{j = 1}^n \left( 1 - \pp\left( \xi_j = 1 \right) \right) \\
        &= 1 - \left( 1 - \frac {1}M \right)^n \\
        &\geq 1  - e^{- \frac {n }M}
    \end{align}
    as desired.
\end{proof}
Unfortunately, Lemma \ref{lem:rejsampling} requires a uniform bound on the likelihood ratio, which is precisely what we are hoping to avoid.  To proceed, we prove a key change of measure lemma:
    \begin{lemma}\label{lem:expectedrestrictedrdbound}
        Let $\mu, \nu$ be probability measures on some set $\cX$ and let $Y \sim \nu$.  For any $f$ satisfying the conditions in Definition \ref{def:fdivergence} and for any $M \geq 2$, the following inequality holds:
        \begin{align}
            \pp\left(\frac{d \nu}{d \mu}(Y) > M\right) \leq \frac{2\cdot \divf{\nu}{\mu}}{f'(M/2)}.
        \end{align}
    \end{lemma}
\begin{proof}[Proof of Lemma \ref{lem:expectedrestrictedrdbound}]
    Note that Lemma \ref{lem:expectedrestrictedrdbound} does not follow from Markov's inequality because we are interested in bounding the probability under $\nu$ that the likelihood ratio is large, while the $f$-divergence is defined as an expectation under $\mu$.  Instead, we observe that for any $s > M > 1$, it holds by convexity that
    \begin{align}
        \frac{f(s)}{s} \geq \frac{f(M)}{M}.
    \end{align}
    Indeed, letting
    \begin{align}
        \ftil(t) = \begin{cases}
            0 & t \leq 1 \\
            f(t) & t \geq 1
        \end{cases},
    \end{align}
    we see that $\ftil$ is convex by our assumptions on $f$.  Moreover, for $s > M > 1$, by convexity of $\ftil$ it holds that
    \begin{align}
        \frac{f(s)}{s} = \frac{\ftil(s) - \ftil(0)}{s} \geq \frac{\ftil(M) - \ftil(0)}{M} = \frac{f(M)}{M}.
    \end{align}

    This implies that $s \leq \frac{M}{f(M)} \cdot f(s)$.  Letting $s = \frac{d\nu}{d\mu}(X)$ and taking an expectation implies that
    \begin{align}
        \nu\left( \frac{d \nu}{d \mu}(Y) > M \right) &=  \ee_\mu\left[ \frac{d\nu}{d \mu}(X) \cdot \bbI\left[ \frac{d\nu}{d\mu}(X) > M \right] \right] \\
        &\leq \frac{M}{f(M)} \cdot \ee\left[ f\left( \frac{d\nu}{d\mu}(X) \right) \cdot \bbI\left[ \frac{d\nu}{d\mu}(X) > M \right] \right] \\
        &\leq \frac{M \cdot \divf{\nu}{\mu}}{f(M)},
    \end{align}
    where the final inequality follows from the assumption that $f \geq 0$.

    We now note that
    \begin{align}
        f(M) = \int_1^M f'(t) d t \geq \int_{M/2}^M f'(t) d t \geq \frac{M}{2} \cdot f'(M/2),
    \end{align}
    where the first inequality follows because $M \geq 2$ and $f'(t) \geq 0$ for $t \geq 1$ and the second inequality follows because $f'$ is nondecreasing.  Rearranging, we see that
    \begin{align}
        \frac{M}{f(M)} \leq \frac{2}{f'(M/2)}.
    \end{align}
    Because $f'(t)$ is non-decreasing, it holds that $\frac{1}{f'(M)} \leq \frac{2}{f'(M/2)}$ as well.  Putting everything together concludes the proof.
\end{proof}
\begin{remark}
    Note that the requirement that $M \geq 2$ is arbitrary.  Indeed, for any $c > 1$, the identical proof shows that for $M \geq c$, $\pp\left( \frac{d\nu}{d\mu}(Y) > M \right) \leq \frac{\divf{\nu}{\mu}}{(c-1) \cdot f'((1-1/c) M)}$.
\end{remark}
We are now ready to prove our main upper bound
\begin{proof}[Proof of \Cref{thm:upperbound}]
    Recall that for any $M \geq 1$, we define $\nuM$ to be a measure on $\cX$ such that
    \begin{align}
        \frac{d \nuM}{d \mu}(X) =  \frac{1}{\ee\left[ \frac{d \nu}{d \mu}(X) \cdot \I\left[ \frac{d \nu}{d \mu}(X) \leq M \right] \right]} \cdot \frac{d \nu}{d \mu}(X) \cdot \I\left[ \frac{d \nu}{d \mu}(X) \leq M \right] .
    \end{align}
    Supposing that $Y \sim \nu$, we see that by construction and Lemma \ref{lem:expectedrestrictedrdbound},
    \begin{align}
        \norm{\frac{d \nuM}{d \mu}}_\infty \leq \frac{M}{\pp\left( \frac{d \nu}{d\mu}(Y) \leq M \right)} \leq \frac{M}{1 - \frac{2 \divf{\nu}{\mu}}{f'(M/2)}} = M'.
    \end{align}
    Thus, by Lemma \ref{lem:rejsampling}, if we run rejection sampling on $\nuM$ with samples from $\mu$ for $n \geq M' \log\left( \frac 1\delta \right)$, it holds with probability at least $1 - \delta$ that we will have at least one accepted sample and that accepted sample will be distributed according to $\nu_M$.  By representing total variation as half the $L^1$ distance between densities, explicit computation tells us that if $E$ is measurable and $\nu^E$ denotes the measure $\nu$ conditioned on the event $E$, then $\tv\left( \nu, \nu^E \right) = \nu(E^c)$.
    Combining this fact with Lemma \ref{lem:expectedrestrictedrdbound} and simplifying, we see that
    \begin{align}
        \tv\left( \nuM, \nu \right) = \pp\left( \frac{d\nu}{d\mu}(Y) > M \right) \leq \frac{2 \cdot \divf{\nu}{\mu}}{f'(M/2)}.
    \end{align}
    Setting
    \begin{align}
        M = 2 \cdot (f')^{-1}\left( \frac{4 \cdot \divf{\nu}{\mu}}{\epsilon} \right)
    \end{align}
    implies then that $\tv(\nu_M, \nu) \leq \frac \epsilon 2$.  Noting that
    \begin{align}
        M' = \frac{M}{1 - \epsilon},
    \end{align}
    concludes the proof.
\end{proof}
We remark here that our proof above actually proves a slightly stronger statement, which is to say that for any $\delta > 0$, there exists a ``success'' event $E$ such that $\mu^{\otimes n}(E) \geq 1 - \delta$ and the law of $X_{\jstar}$ conditioned on the event $E$ has total variation distance at most $\epsilon$ from the target measure $\nu$, as long as $\delta, \epsilon, n$ are such that
\begin{align}
    n \geq \frac{2}{1 - \epsilon}\log\left( \frac{1}{\delta} \right) (f')^{-1}\left( \frac{4 \cdot \divf{\nu}{\mu}}{\epsilon} \right).
\end{align}
In fact, it is this formulation of the upper bound that will be useful in our applications.

\subsection{Lower Bound}
In this section, we prove our main lower bound, \Cref{thm:lowerbound}, as well as the alternative version, \Cref{thm:lowerboundweakf}.  We begin by observing that any selection rule $\jstar$ has a Radon-Nikodym derivative that is bounded above with respect to $\mu$:
\begin{lemma}\label{lem:jstarrdboundedbyn}
    Suppose that $X_1, \dots, X_n \sim \mu$ are independent and let $\jstar$ be a selection rule such that $P_{X_{\jstar}} = \nutil$.  Then it holds that
    \begin{align}
        \norm{\frac{d \nutil}{d \mu}}_{L^\infty(\mu)} \leq n.
    \end{align}
\end{lemma}
\begin{proof}
    Let $A \in \scrF$ be a measurable set and let $\nutil$ denote the law of $X_{\jstar}$.  We then observe
    \begin{align}
        \nutil\left( A \right) &= \sum_{j = 1}^n \pp\left( X_j \in A \text{ and } \jstar = j \right) \\
        &= \sum_{j  =1}^n \mu(A) \cdot \pp\left( \jstar = j | X_j \in A \right) \\
        &= \mu(A) \cdot \sum_{j  =1}^n \pp\left( \jstar = j | X_j \in A \right) \\
        &\leq n \cdot \mu(A).
    \end{align}
    The above computation holds for all $A \in \scrF$ and so the result holds.
\end{proof}

With Lemma \ref{lem:jstarrdboundedbyn} proved, we see that it suffices to turn our attention to those distributions with bounded likelihood ratios with respect to the base measure.  To prove our lower bounds, we will separate our analysis into two cases.  The easier case is that of linear $f$, i.e., those $f$ with bounded $f'$.  We have the following bound:
\begin{lemma}\label{lem:linearf}
    Suppose that $f$ is a convex function as in Definition \ref{def:fdivergence} satisfying
    \begin{align}
        \lim_{t \to \infty} f'(t) = C < \infty.
    \end{align}
    For any $0 < \Delta \leq C + f(0)$, there exists some $\epsilon > 0$ depending only on $C, \Delta$, and $f$ such that there exist distributions satisfying $\divf{\nu}{\mu} = \Delta$ and
    \begin{align}
        \inf_{\nutil \text{ such that } \norm{\frac{d \nutil}{d\mu}} < \infty} \tv\left( \nutil, \nu \right) > \epsilon.
    \end{align}
\end{lemma}
\begin{proof}
    Fix some nonatomic $\nu$ and, for some $\epsilon < 1$ to be determined, fix $A \in \scrF$ such that $\nu(A) = \epsilon$.  Define $\mu$ such that
    \begin{align}
        \frac{d \mu}{d \nu}(X) = \I\left[ Z \not\in A \right] \frac{1}{1 - \epsilon}.
    \end{align}
    Observe that this defines a valid likelihood ratio and note that by definition,
    \begin{align}
        \divf{\nu}{\mu} = f(1 - \epsilon) + \epsilon f'(\infty) = f(1 - \epsilon) + \epsilon C.
    \end{align}
    As $f$ is continuous as $\epsilon \downarrow 0$ and $f(1) = 0$, it holds that  $\divf{\nu}{\mu}$ traverses continuously between $C + f(0)$ and 0 as $\epsilon$ moves between $0$ and $1$.  Thus, the intermediate value theorem tells us that there exists some $\epsilon$ such that $\divf{\nu}{\mu} = \Delta$.  For this $\epsilon$, we note that $\nutil(A) = 0$ for any $\nutil$ that is absolutely continuous with respect to $\mu$.  Thus,
    \begin{align}
        \tv\left( \nutil, \nu \right) \geq \abs{\nutil(A) - \nu(A)} = \epsilon > 0
    \end{align}
    and the result follows.
\end{proof}
We can now state and prove the formal version of Proposition \ref{prop:lowerboundlinearinformal}:
\begin{theorem}\label{thm:lowerboundlinearformal}
    Suppose that $f$ is a convex function as in Definition \ref{def:fdivergence} satisfying
    \begin{align}
        \lim_{t \to \infty} f'(t) = C < \infty.
    \end{align}
    For any $0 < \Delta, \leq C + f(0)$, there exists some $\epsilon > 0$ depending only on $C, \Delta,$ and $f$ such that there exist distributions $\mu, \nu$ with $\divf{\nu}{\mu} = \Delta$ satisfying the following property: if $X_1, \dots, X_n \sim \mu$ are independent, then
    \begin{align}
        \inf_{\jstar} \tv\left( P_{X_{\jstar}}, \nu \right) \geq \epsilon
    \end{align}
    where the infimum is over all selection rules $\jstar$.
\end{theorem}
\begin{proof}
    By Lemmas \ref{lem:jstarrdboundedbyn} and \ref{lem:linearf}, it holds that
    \begin{align}
        \inf_{\jstar} \tv\left( P_{X_{\jstar}}, \nu \right) \geq \inf_{\nutil \text{ such that } \norm{\frac{d \nutil}{d\nu}} < \infty} \tv\left( \nutil, \nu \right) \geq \epsilon.
    \end{align}
    The result follows.
\end{proof}
Note that \Cref{thm:lowerboundlinearformal} implies that if $\divf{\cdot}{\cdot}$ is too coarse a notion of similarity then approximate sampling from $\nu$ using $\mu$ can be impossible.

We turn now to the more complicated case, that of superlinear $f$.  We begin by showing that the $\cE_\gamma$-divergence, defined in Example \ref{ex:Egamma}, provides a lower bound on the total variation between the law of a selection rule and the target measure:
\begin{lemma}\label{lem:egammalowerbound}
    Let $\mu, \nu$ be measures and for some $\gamma \geq 1$, let $\cE_\gamma$ denote the divergence given in Example \ref{ex:Egamma}.  Then it holds that
    \begin{align}
        \inf_{\norm{\frac{d \nutil}{d \mu}}_{L^\infty(\mu)} \leq \gamma} \tv\left( \nutil, \nu \right) \geq \cE_\gamma(\nu || \mu ).
    \end{align}
\end{lemma}
\begin{proof}
    Observe that for any $\nutil$ satisfying $\norm{\frac{d \nutil}{d\nu}}_\infty \leq \gamma$, it holds that
    \begin{align}
        \tv\left( \nutil, \nu \right) &= \sup_{A \in \scrF} \abs{\nu(A) - \nutil(A)} \\
        &= \sup_{A \in \scrF} \abs{\ee\left[ \I[X \in A] \left( \frac{d \nu}{d \mu}(X) - \frac{d \nutil}{d\mu}(X) \right) \right]} \\
        &\geq \ee\left[ \I\left[ \frac{d \nu}{d\mu}(X) > \gamma \right] \left( \frac{d \nu}{d\mu}(X) - \gamma \right) \right] \\
        &= \ee\left[ \left( \frac{d \nu}{d\mu}(X) - \gamma \right)_+ \right] \\
        &= \cE_\gamma\left( \nu || \mu \right).
    \end{align}
\end{proof}
We are now prepared to prove Theorem \ref{thm:lowerbound}:
\begin{proof}[Proof of Theorem \ref{thm:lowerbound}]
    We observe that by comining Lemmas \ref{lem:jstarrdboundedbyn} and \ref{lem:egammalowerbound}, it suffices to exhibit two measures $\mu, \nu$ such that $\cE_n(\nu || \mu) > \epsilon$ and $\divf{\nu}{\mu}$ is bounded.  We let $\mu = \ber(q)$ and $\nu = \ber(p)$, for
    \begin{align}
        q = \frac \epsilon n, && p = 2 \epsilon.
    \end{align}
    We observe that
    \begin{align}
        \frac{1 - p}{1 - q} = \frac{1 - 2 \epsilon}{1 - \frac\epsilon n} \in \left[ \frac 12, 1 \right]
    \end{align}
    by the assumption that $\epsilon \leq \frac 14$.  Thus, for $n \geq 1$, we see that
    \begin{align}
        \cE_n(\nu || \mu) = q\left( \frac pq - n\right) = \epsilon.
    \end{align}
    We similarly compute that
    \begin{align}
        \divf{\nu}{\mu} = q \cdot f\left( \frac{p}{q} \right) + (1 - q) f\left( \frac{1 - p}{1 - q} \right) \leq \frac{\epsilon}{n} \cdot f(2n) + f\left( \frac 12 \right),
    \end{align}
    where the inequality follows by the convexity of $f$ and the above computation.  We now observe that by convexity of $f$, we have
    \begin{align}
        0 = f(1) \geq f(x) + (1 - x) f'(x)
    \end{align}
    implies that $\frac{f(x)}{x - 1} \leq f'(x)$ for all $x > 1$.  Thus,
    \begin{align}
        \divf{\nu}{\mu} \leq 2 \epsilon f'(2n) + f\left( \frac 12 \right).
    \end{align}
    Note that if
    \begin{align}
        n < \frac 12 (f')^{-1}\left( \frac{\delta}{2 \epsilon} \right),
    \end{align}
    then $\divf{\nu}{ \mu} \leq \delta$ and so we have exhibited a pair $(\nu, \mu)$ as desired.  The result follows.
\end{proof}

We turn now to the proof of \Cref{thm:lowerboundweakf}.  As stated in the main paper, the method is similar up to the point of exhibiting distributions $\mu, \nu$ with large $\cE_n\left( \nu || \mu \right)$.  We have the following result:
\begin{lemma}\label{lem:lowerboundsuperlinear}
    Suppose that $f$ is a convex function as in Definition \ref{def:fdivergence} such that $f(0) < \infty$ and for some $\zeta > 0$, the function
    \begin{align}
        t \mapsto \frac{t f''(t)}{(f'(t))^{1 + \zeta}}
    \end{align}
    is decreasing for sufficiently large $t$ and tends to 0 as $t\uparrow\infty$.  Then for sufficiently large $n$, there exist distributions $\mu, \nu$ such that $\divf{\nu}{\mu} < \infty$ and 
    \begin{align}
        \cE_n(\nu || \mu) \geq \frac 18 \cdot \left( \frac{\zeta \divf{\nu}{\mu}}{f'(n)} \right)^{1 + \zeta}.
    \end{align}
\end{lemma}

\begin{proof}
    We first provide a rough intuition for the result. Ideally, we want to construct a pair of
    distributions such that $D_f(\nu\|\mu)<\infty$ and for all $\gamma \ge \gamma_0$ we have
    $$ \cE_\gamma(\nu\|\mu) = \frac C{f'(\gamma)}\,.$$
    Note that \citet[(2.144)]{YP10} shows that $\frac d{d\gamma}\cE_\gamma(\nu\|\mu) = -\pp\left(\frac{d\nu}{d\mu}(X) \geq \gamma\right)$, where $X\sim \mu$. Thus, we see that our goal is to construct a pair of distributions such that
    \begin{align}
        \pp\left(\frac{d\nu}{d\mu}(X) \geq \gamma\right) = C \cdot\frac{ f''(\gamma)}{(f'(\gamma))^2}.
    \end{align}
    Though by a simple change of variable we see that the expression on the right-hand side is
    integrable at $\infty$, it does not have to be monotone (hence the extra assumption on $f$).
    Even if it is monotone, however, it can be seen that $D_f(\nu\|\mu)$ is not finite. Thus, we
    slightly tweak this construction below.  

    Fix $\delta, \beta, \zeta > 0$ with $\delta < n$ and let
    \begin{align}
        F(t) = \begin{cases}
            0 & t < 0 \\
            1 - \frac{\beta f''\left( (f')^{-1}(\delta) \right)}{\delta^{2 + \zeta}} & 0 \leq t < (f')^{-1}(\delta) \\
            1 -  \frac{ \beta f''(t)}{(f'(t))^{2 + \zeta}} & t \geq (f')^{-1}(\delta)
        \end{cases}.
    \end{align}
    We claim that $F$ is a valid cumulative distribution function for properly chosen $\beta, \delta, \zeta> 0$.  To begin with, we note that $F$ is right continuous by construction.  It is similarly clear that $F(t) \downarrow 0$ as $t \downarrow 0$.  If $\delta$ is sufficiently large such that
    \begin{align}
        t \mapsto \frac{t f''(t)}{f'(t)^{2 + \zeta}}
    \end{align}
    is decreasing for all $t > \delta$ (such a $\delta$ always exists by the assumption in the
    statement), we see that $F(t)$ is nondecreasing.  Finally, to see that $F(t) \uparrow 1$ as $t
    \uparrow \infty$, note that
    \begin{align}
        \int_{(f')^{-1}(\delta)}^\infty \frac{\beta f''(t)}{f'(t)^{2+\zeta}} d t &= \lim_{N \uparrow \infty} \left(\frac{\beta}{(1 + \zeta) \delta^{1 + \zeta}} - \frac{\beta}{(1 + \zeta) (f')^{-1}(N)^{1 + \zeta}} \right) \\
        &= \frac{\beta}{(1 + \zeta)\delta^{1 + \zeta}}
    \end{align}
    by the assumption that $f'(t) \uparrow \infty$ as $t\uparrow \infty$.  In particular, it holds that
    \begin{align}
        \lim_{t \uparrow \infty} \frac{\beta f''(t)}{f'(t)^{2+\zeta}} = 0
    \end{align}
    and so $F(t)$ is a cumulative distribution function.  Note further that if a random variable on the nonnegative real line $Z$ is distributed according to $F$, then by Fubini's theorem,
    \begin{align}
        \ee\left[ Z \right] &= \int_{(f')^{-1}(\delta)}^\infty \beta  \frac{ f''(t)}{(f'(t))^{2 + \zeta}} d t  \\
        &= \frac{\beta}{(1 + \zeta)\delta^{1 + \zeta}}.
    \end{align}
    Thus if $\beta = (1 + \zeta)\delta^{1+\zeta}$ then $\ee[Z] = 1$.  Thus with this choice of $\beta$, we let $\mu$ be nonatomic on some set $\cX$ and let $\frac{d\nu}{d\mu}(X)$ be distributed according to $F$, where $X \sim \mu$.  We first compute the $f$-divergence between $\nu$ and $\mu$ using Fubini's theorem:
    \begin{align}
        \divf{\nu}{\mu} &= \ee\left[ f\left( \frac{d \nu}{d\mu}(X) \right) \right] \\
        &= \ee\left[ f(Z) \right] \\
        &= f(0) \pp\left( Z = 0 \right) + \int_{(f')^{-1}(\delta)}^\infty f'(t) \pp\left( Z > t \right) d t \\
        &=  f(0) \left( 1 - \frac{\beta f''\left( (f')^{-1}(\delta) \right)}{\delta^{2+\zeta}} \right) + \int_{(f')^{-1}(\delta)}^\infty \frac{\beta f''(t)}{f'(t)^{1+\zeta}} d t \\
        &=f(0) \left( 1 - \frac{\beta f''\left( (f')^{-1}(\delta) \right)}{\delta^{2+\zeta}} \right) + \frac{\beta}{\zeta \delta^\zeta} \\
        &\leq f(0) + \frac{1 + \zeta}{\zeta} \cdot \delta,
    \end{align}
    where we used the fact that the second derivative of a convex function is nonnegative and out computation of $\beta = (1 + \zeta) \delta^{1 + \zeta}$ above.  If we take $\delta, \zeta$ such that $f(0) \leq \frac{1 + \zeta}{\zeta} \cdot \delta$, then we have
    \begin{align}
        \divf{\nu}{\mu} \leq 2\frac{1 + \zeta}{\zeta} \delta.
    \end{align}
    Again by Fubini's theorem, using the fact that $n > \delta$, we see that
    \begin{align}
        \ee\left[ \frac{d\nu}{d\mu} \I\left[ \frac{d\nu}{d\mu} > n \right] \right] &= \int_n^\infty \frac{\beta f''(t)}{(f'(t))^{2+\zeta}} d t \\
        &= \frac{\beta}{(1 + \zeta) f'(n)^{1 + \zeta}}.
    \end{align}
    Finally, note that
    \begin{align}
        n \pp\left( \frac{d\nu}{d\mu} > n \right) = \frac{n \beta f''(n)}{f'(n)^{2 + \zeta}}.
    \end{align}
    Putting everything together, we see that
    \begin{align}
        \cE_n\left( \nu || \mu \right) &\geq \frac{\beta}{(1 + \zeta) f'(n)^{1 + \zeta}} - \frac{n \beta f''(n)}{f'(n)^{2 + \zeta}} \\
        &= \frac{\delta^{1 + \zeta}}{f'(n)^{1 + \zeta}} - \frac{n (1 + \zeta) \delta^{1 + \zeta} f''(n)}{f'(n)^{2 + \zeta}} \\
        &= \left( \frac
        \delta{f'(n)} \right)^{1+ \zeta} \cdot \left( 1 - \frac{n (1 + \zeta) f''(n)}{f'(n)} \right) \\
        &\geq \left(\frac{\zeta}{2(1+\zeta)} \cdot \frac{\divf{\nu}{\mu}}{f'(n)} \right)^{1 + \zeta} \cdot \left( 1 - \frac{n (1 + \zeta) f''(n)}{f'(n)} \right).
    \end{align}
    By the assumption in the statement, for sufficiently large $n$, we have
    \begin{align}
        \frac{n f''(n)}{f'(n)^{1 + \zeta}} < \frac{1}{4}
    \end{align}
    and thus the result holds.
\end{proof}
We remark that the requirement that $f(0) < \infty$ is relatively weak, but does not hold for some important $f$-divergences.  This requirement, however, could be easily removed by placing the atom at $\frac 12$ instead of at 0 in the above proof; for the sake of simplicity we do not expand on this here, as it leads to a more intricate proof with little additional clarity.  We now state and prove a formal version of \Cref{thm:lowerboundweakf}:
\begin{theorem}\label{thm:lowerboundsuperlinearformal}
    Let $f$ be a convex function as in Definition \ref{def:fdivergence} with $f(0) < \infty$ and let $\zeta > 0$ be arbitrary.  Suppose that there is some $t_0 > 0$ such that the function
    \begin{align}\label{eq:growthcondition}
        t \mapsto \frac{t f''(t)}{(f'(t))^{1+\zeta}}
    \end{align}
    is non-increasing for all $t > t_0$ and
    \begin{align}
        \lim_{t \uparrow \infty} \frac{t f''(t)}{(f'(t))^{1 + \zeta}} = 0.
    \end{align}
    Then there exist distributions $\mu, \nu$ with $\divf{\nu}{\mu} < \infty$ such that if $X_1, \dots, X_n \sim \mu$ are independent then
    \begin{align}
        \inf_{\jstar} \tv\left( P_{X_{\jstar}}, \nu \right) \geq \frac 18 \left( \frac{\zeta \divf{\nu}{\mu}}{f'(n)} \right)^{1 + \zeta}
    \end{align}
    where the infimum is over all selection rules.
\end{theorem}
\begin{proof}
    The result follows by combining Lemmas \ref{lem:jstarrdboundedbyn}
    ,\ref{lem:egammalowerbound}, and \ref{lem:lowerboundsuperlinear}.
\end{proof}
We observe that for essentially all common suplerlinear $f$-divergences, such as KL-divergence and Renyi divergences, the assumptions on the function defined in \eqref{eq:growthcondition} hold.

\section{Proofs from Section \ref{sec:smoothedonline}}\label{app:smoothedonline}
\subsection{Proof of Proposition \ref{prop:tvlowerbound}}

Let $\mu$ denote the uniform measure on the unit interval and for each $t$, let $p_t = (1 - \delta) \mu + \delta q_t$ for some $q_t$ to be defined.  Suppose that $f'(\infty) = C < \infty$.  Note that independent of $q_t$, it holds that 
\begin{align}
    \divf{\nu}{\mu} \leq \delta f'(\infty) = \delta C
\end{align}
Thus if the adversary samples $x_t$ from $p_t$, and $\delta \leq \sigma / C$, the adversary is $(f, \sigma)$-smooth.  Now, define $\bar{x}_t$ for $1 \leq t < \infty$ as follows.  Let $\bar{x}_1 = \frac 12$ and let $\epsilon_1, \epsilon_2, \dots$ be independent Rademacher random variables.  Let
\begin{align}
    \bar{x}_t = \frac 12 + \sum_{s = 1}^{t-1} \epsilon_s 2^{-s-1}
\end{align}
and note that $\bar{x}_t \in [0,1]$ for all $t$ almost surely.  Furthermore, note that $\bar{x}_t \to \bar{x}_\infty$ almost surely and define $\theta^\ast = \bar{x}_\infty$.  Let $y_t = \I[x_t \geq \theta^\ast]$ for all $t$ and note that this adversary is realizable, i.e., there exists some $f \in \cF$ that attains zero regret.  Let $q_t$ denote an atom at $\bar{x}_t$ and suppose that the adversary plays $x_t \sim p_t$.  As mentioned above, this adversary is $f$-smooth.  Note that whenever $x_t = \bar{x}_t$, it holds that $y_t = \frac{1 - 2 \epsilon_t}{2}$.  By independence of $\epsilon_t$, then, it holds for any $T$ that
\begin{align}
    \ee\left[ \sum_{t = 1}^T \I[y_t \neq \yhat_t] \right] &\geq  \sum_{t = 1}^T \pp\left( x_t = \bar{x}_t \right) \pp\left( \yhat_t \neq y_t | x_t = \bar{x}_t \right) \\
    &= \sum_{t = 1}^T \delta \cdot \frac 12 \\
    &= \frac{\delta T}{2}.
\end{align}
The result follows.
\subsection{Proof of Lemma \ref{lem:coupling}}
We proceed as in \citet{haghtalab2022smoothed,block2022smoothed}, but apply our \Cref{thm:upperbound} instead of the standard rejection sampling bound.  We begin by sampling $Z_{t,j}$ independently for all $1 \leq t \leq T$ and $1 \leq j \leq n$.  Applying \Cref{thm:upperbound} on the distribution of $x_t$ conditioned on the history and then using Definition \ref{def:weaksmooth} to bound $\divf{P_{x_t}}{\mu}$ concludes the proof.

\subsection{Minimax Regret for $f$-Smoothed Online Learning}
In this section we prove a generalization of \Cref{thm:minimax} to arbitrary function classes.  We follow the proof technique of \citet{block2022smoothed} with the exception of using our new rejection sampling coupling from Lemma \ref{lem:coupling}.  We have the following result:
\begin{theorem}\label{thm:minimaxgeneral}
    Let $\cF: \cX \to [-1,1]$ be a real-valued function class and let $\vc{\cF, \alpha}$ denote its scale-sensitive VC dimension.  Suppose that $\ell: [-1,1] \times [-1,1] \to [0,1]$ is a loss function that is Lipschitz in the first argument.  Further, let $f$ be a convex function as in Definition \ref{def:fdivergence} such that $f'(\infty) = \infty$.  If an adversary is $(f, \sigma)$-smooth in the sense of \Cref{sec:smoothedonline}, then there are universal constants $c, C > 0$ such that there exists an algorithm with $\ee\left[ \reg_T \right]$ bounded above by the following expression:
    \begin{align}
        C\cdot \inf_{\beta, \gamma, \epsilon > 0} (\epsilon + \gamma) T + \sqrt{T \log^{1 + \beta}\left( T(f')^{-1}\left( \frac{1}{\sigma \epsilon} \right) \right)} \int_\gamma^1 \sqrt{\vc{\cF, c \beta \delta} \cdot \log^{1 + \beta}\left( \frac{1}{\vc{\cF, c \delta} \delta} \right)} d \delta.
    \end{align}
\end{theorem}
\begin{proof}
    Applying \Cref{thm:seqradcomplexity}, we see that it is enough to control $\radseq_T(\cF, \scrD)$, where $\scrD$ is the class of $(f, \sigma)$-smooth adversaries.  Fix some $\frac 12 \geq \alpha, \delta > 0$ and let $\Pi$ denote a coupling between $x_1, \dots, x_T$ and $\left\{ Z_{t,j}| 1 \leq t \leq T, 1 \leq j \leq n \right\}$ guaranteed by Lemma \ref{lem:coupling} such that
    \begin{align}
        n \geq 2 \log\left( \frac T\delta \right) (f')^{-1}\left( \frac{1}{\alpha \sigma} \right),
    \end{align}
    the $Z_{t,j}\sim \mu$ are independent, and with probability at least $1 - \delta$, there are selection rules $\jstar_t$ such that $\tv\left( P_{x_t},  P_{Z_{t,\jstar_t}}\right) \leq \alpha$.  Denote by $\cE$ the event that these selection rules exist and note that $\pp(\cE) \geq 1 - \delta$.  We now fix some $P_{x_{1:T}}$ and compute:
    \begin{align}
        \radseq_T\left( \cF, P_{x_{1:T}} \right) &= \ee_{\substack{\rho_{P_{x_{2:T}}} \\ \epsilon}}\left[ \sup_{g \in \cF} \sum_{t = 1}^T \epsilon_t g(\bx_t(\epsilon)) \right] \\
        &= \ee_{\Pi, \epsilon}\left[ \sup_{g \in \cF} \sum_{t = 1}^T \epsilon_t g(\bx_t(\epsilon)) \right] \\
        &= \ee_{\Pi, \epsilon}\left[\I[\cE] \sup_{g \in \cF} \sum_{t = 1}^T \epsilon_t g(\bx_t(\epsilon)) \right] + \ee_{\Pi, \epsilon}\left[\I\left[ \cE^c \right] \sup_{g \in \cF} \sum_{t = 1}^T \epsilon_t g(\bx_t(\epsilon)) \right].
    \end{align}
    For the second term, note that almost surely,
    \begin{align}
        \sup_{g \in \cF} \sum_{t = 1}^T \epsilon_t g(\bx_t(\epsilon)) \leq T
    \end{align}
    and thus
    \begin{align}\label{eq:minimaxproof1}
        \ee_{\Pi, \epsilon}\left[\I\left[ \cE^c \right] \sup_{g \in \cF} \sum_{t = 1}^T \epsilon_t g(\bx_t(\epsilon)) \right] \leq \delta T.
    \end{align}
    For the other term, we observe that
    \begin{align}
        \ee_{\Pi, \epsilon}\left[\I[\cE] \sup_{g \in \cF} \sum_{t = 1}^T \epsilon_t g(\bx_t(\epsilon)) \right] &\leq \ee_{\Pi, \epsilon}\left[\I[\cE] \sup_{g \in \cF} \sum_{t = 1}^T \epsilon_t g(\bx_t(\epsilon)) - \epsilon_t g\left( Z_{t,\jstar_t} \right) \right] \\
        &+ \ee_{\Pi, \epsilon}\left[\I[\cE] \sup_{g \in \cF} \sum_{t = 1}^T \epsilon_t g\left( Z_{t,\jstar_t} \right) \right].
    \end{align}
    For thefirst term, we see that
    \begin{align}
        \ee_{\Pi, \epsilon}\left[\I[\cE] \sup_{g \in \cF} \sum_{t = 1}^T \epsilon_t g(\bx_t(\epsilon)) - \epsilon_t g\left( Z_{t,\jstar_t} \right) \right] &\leq \sum_{t =1}^T \ee_{\Pi, \epsilon}\left[ \sup_{g \in \cF} \epsilon_t g(\bx_t(\epsilon)) - \epsilon_t g\left( Z_{t,\jstar_t} \right) \right] \\
        &\leq T \max_{1 \leq t \leq T} \tv\left( P_{\bx_t(\epsilon)}, P_{Z_{t,\jstar_t}} \right) \\
        &\leq \alpha T \label{eq:minimaxproof2}
    \end{align}
    by construction.  For the second term, we use Jensen's inequality and the tower property of conditional expectations to compute:
    \begin{align}
        \ee_{\Pi, \epsilon}\left[\I[\cE] \sup_{g \in \cF} \sum_{t = 1}^T \epsilon_t g\left( Z_{t,\jstar_t} \right) \right] &\leq \ee_{\Pi, \epsilon}\left[ \sup_{g \in \cF} \sum_{t = 1}^T \epsilon_t g\left( Z_{t,\jstar_t} \right) \right] \\
        &= \ee_{Z_{t,j}}\left[\ee_{\epsilon}\left[ \sup_{g \in \cF} \sum_{t = 1}^T \epsilon_t g\left( Z_{t,\jstar_t} \right) | \left\{ Z_{t,j} \right\} \right] \right] \\
        &= \ee_{Z_{t,j}}\left[\ee_{\epsilon}\left[ \sup_{g \in \cF_{\left\{ Z_{t,j} \right\}}} \sum_{t = 1}^T \epsilon_t g\left( Z_{t,\jstar_t} \right) | \left\{ Z_{t,j} \right\} \right] \right].
    \end{align}
    Noting that $\abs{\left\{ Z_{t,j} \right\}} = T n$, we may now apply Lemma \ref{lem:minimaxregretfinitedomain} to conclude that this last display is upper bounded by:
    \begin{align}\label{eq:minimaxproof3}
        \inf_{\substack{\beta, \gamma > 0}} \gamma T + \sqrt{T \cdot \log^{1 + \beta}(nT)} \int_\gamma^1 \sqrt{\vc{\cF, c \beta \delta} \cdot \log^{1+\beta}\left( \frac{1}{\vc{\cF, c \delta} \delta} \right)}d \delta.
    \end{align}
    Setting $\delta = \frac 1T$ and combining \eqref{eq:minimaxproof1}, \eqref{eq:minimaxproof2}, and \eqref{eq:minimaxproof3} concludes the proof.
\end{proof}

\subsection{Oracle-Efficient Algorithms}
In this section, we turn to computationally tractable algorithms.  In particular, we are interested in algorithms that make only polynomially many calls to an Empirical Risk Minimization (ERM) oracle defined below.  Note that ERM oracles are common models of computational access in the online learning community \citep{kalai2005efficient,hazan2016computational,block2022smoothed,haghtalab2022oracle} due both to the fact that they suffice for learning in the statistical setting (where data appear independently) and because there are popular computational heuristics for implementing these oracles in many problems of interest.  We will consider two algorithms: an improper algorithm requiring two oracle calls per round achieving regret that scales with the Rademacher complexity (see \eqref{eq:rademachercomplexity}) and a proper algorithm requiring one oracle call per round.  Both of the algorithms were proposed in \citet{block2022smoothed} and we use a similar analysis to bound their regret, with the modification of replacing the coupling from \citet{block2022smoothed} with our more general version, Lemma \ref{lem:coupling}.  We begin by defining the ERM oracle:
\begin{definition}\label{def:ermoracle}
    We assume that the learner has access to $\ermoracle$, which, given a set of tuples $(x_1, y_1), \dots, (x_m, y_m) \in \cX \times [-1,1]$, a list of weights $w_1, \dots, w_m \in \rr$, and a sequence of $[0,1]$-valued loss functions $\ell_1, \dots, \ell_m$, returns some $\ghat \in \cF$ satisfying
    \begin{align}
        \sum_{i = 1}^m w_i \ell_i\left( \ghat(x_i), y_i \right) \leq \inf_{g \in \cF} \sum_{i = 1}^m w_i \ell_i(f(x_i), y_i).
    \end{align}
\end{definition}
A slightly weaker assumption allows for some approximation, where $\ermoracle$ returns some $\ghat$ that is $\epsilon$-close to the actual minimizer.  For the sake of simplicity, we restrict our focus to exact oracles here, but all of our results apply to the more general setting up to an additive $\epsilon T$ with essentially no modification of the proofs.

\subsubsection{Improper Algorithm through Relaxations}
We now turn to the first oracle-efficient algorithm, motivated by the relaxations framework of \citet{rakhlin2012relax}.  We closely follow the presentation of \citet{block2022smoothed}.  To begin, we define a relaxation:
\begin{definition}\label{def:relaxation}
    For a fixed horizon $T$, function class $\cF$, context space $\cX$, and measure $\mu$, we say that a sequence of relaxations $\relT{t}: \cX^{\times t} \times [-1,1]^{\times t} \to \rr$ is a relaxation if for any sequence $x_1, \dots, x_T$ and for any $1 \leq t \leq T$, the following two properties hold:
    \begin{align}
        - \inf_{g \in \cF} \sum_{t =1}^T \ell(g(x_t), y_t) &\leq \relT{T} \\
        \sup_{p_t} \ee_{x_t \sim p_t} \inf_{q_t \in \Delta([-1,1])} \sup_{y_t' \in[-1,1]} \left\{ \ee_{\yhat_t\sim q_t}[\ell(\yhat_t, y_t')] + \rel\left( \cF | x_1,y_1 \dots, x_t', y_t' \right) \right\} &\leq \relT{t-1}
    \end{align}
    where the first supremum is over $(f,\sigma)$-smooth distributions with respect to $\mu$ and infimum is over distributions on $[-1,1]$.
\end{definition}
The key property of relaxations, as proven in \citet[Proposition 1]{rakhlin2012relax}, is that any strategy $q_t$ that guarantees the second inequality in Definition \ref{def:relaxation} achieves regret bounded above by $\rel(\cF | \emptyset)$.  Our first result shows that, with minor modifications, the relaxation proposed in \citet{block2022smoothed} remains a valid relaxation in the $f$-smoothed regime.
\begin{lemma}\label{lem:validrelaxation}
    Suppose that the adversary is $(f, \sigma)$-smoothed and that the loss function $\ell$ is convex and Lipschitz in the first argument.  Let $0 < \alpha \leq \frac 12$ and suppose that
    \begin{align}
        n \geq 8 \log(T) (f')^{-1}\left( \frac{1}{\alpha \sigma} \right).
    \end{align}
    Then
    \begin{align}
        \relT{t} = \ee_{\mu, \epsilon}\left[ 2  \sup_{g \in \cF} \sum_{j = 1}^n \sum_{s = t+1}^T \epsilon_{s,j} g(Z_{s,j}) - \sum_{s= 1}^t \ell(g(x_s), y_s) \right] + T \alpha + \frac{T - t}{n T}
    \end{align}
    is a relaxation, where the expectation is with respect to $Z_{s,j} \sim \mu$ and independent Rademacher random variables.
\end{lemma}
\begin{proof}
    We follow the proof technique of \citet[Proposition 6]{block2022smoothed}, replacing their \citet[Lemma 14]{block2022smoothed} with our Lemma \ref{lem:coupling}.  We introduce the same convenient shorthand:
    \begin{equation}
        L_t(g) = \sum_{s=  1}^t \ell(g(x_s), y_s)
    \end{equation}
    for all $g \in \cF$.  We now note that the first condition of Definition \ref{def:relaxation} is immediate as $\relT{T}$ is in fact equal to the infimum on the left hand side of the defining inequality.  Thus it suffices to demonstrate that for all $t$ and all realizations $x_1, y_1, \dots, x_{t-1}, y_{t-1}$, the second inequality in Definition \ref{def:relaxation} holds.  We fix some $(f, \sigma$)-smoothed $p_t$ and argue for this arbitrary smoothed distribution.  Because the loss function $\ell$ is convex in the first argument, we may replace the distribution $q_t$ from which $\yhat$ is sampled by its expectation, i.e.,
    \begin{align}
        \inf_{q_t \in \Delta([-1,1])} &\sup_{y_t' \in[-1,1]} \left\{ \ee_{\yhat_t\sim q_t}[\ell(\yhat_t, y_t')] + \rel\left( \cF | x_1,y_1 \dots, x_t', y_t' \right) \right\} \\
        &= \inf_{\yhat_t \in [-1,1]} \sup_{y_t' \in[-1,1]} \left\{ \ell(\yhat_t, y_t') + \rel\left( \cF | x_1,y_1 \dots, x_t', y_t' \right) \right\}.
    \end{align}
    Now, arguing as in \citet{block2022smoothed}, we see that for any $x_t' \in \cX$,
    \begin{align}
        \inf_{\yhat_t} &\sup_{y_t'}\left\{ \ell(\yhat_t, y_t') + \ee_{\mu,\epsilon}\left[ \sup_{g \in \cF} 2\sum_{j = 1}^n \sum_{s = t+1}^T \epsilon_{s,j}g(Z_{s,j}) - L_t(g) \right]  \right\} \\
        &= \inf_{\yhat_t} \sup_{y_t'}\ee_{\mu, \epsilon}\left\{ \sup_{g \in \cF}  2\sum_{j = 1}^n \sum_{s = t+1}^T \epsilon_{s,j}g(Z_{s,j}) - L_{t-1}(g) + \ell(\yhat_t, y_t') - \ell(g(x_t'), y_t') \right\}\\
        &\leq \inf_{\yhat_t} \sup_{y_t'}\ee_{\mu, \epsilon}\left\{ \sup_{g \in \cF}  2\sum_{j = 1}^n \sum_{s = t+1}^T \epsilon_{s,j}g(Z_{s,j}) - L_{t-1}(g) + \partial \ell(\yhat_t, y_t')\left( \yhat_t - g(x_t') \right) \right\} \\
        &\leq \inf_{\yhat_t} \max_{\gamma_t \in \{\pm 1\}}\ee_{\mu, \epsilon}\left\{ \sup_{g \in \cF}  2\sum_{j = 1}^n \sum_{s = t+1}^T \epsilon_{s,j}g(Z_{s,j}) - L_{t-1}(g) +  \gamma_t (\yhat_t - g(x_t'))\right\},
    \end{align}
    where the first inequality follows from the fact that $\ell$ is convex in the first argument and the second inequality follows from the fact that $\ell$ is Lipschitz in the same.  We now apply the minimax theorem, where we are forced to invoke a supremum over distributions on $\{\pm 1\}$ due to the lack of convexity of this set.  Following again the argument of \citet{block2022smoothed}, we let $d_t$ denote a distribution on $\{\pm 1\}$ and sample $\gamma_t \sim d_t$.  We compute:
    \begin{align}
        \inf_{\yhat_t} &\max_{\gamma_t \in \{\pm 1\}}\ee_{\mu, \epsilon}\left\{ \sup_{g \in \cF}  2\sum_{j = 1}^n \sum_{s = t+1}^T \epsilon_{s,j}g(Z_{s,j}) - L_{t-1}(g) +  \gamma_t (\yhat_t - g(x_t'))\right\} \\
        &= \sup_{d_t}\inf_{\yhat_t}\ee_{\mu, \epsilon, d_t}\left\{ \sup_{g \in \cF}  2\sum_{j = 1}^n \sum_{s = t+1}^T \epsilon_{s,j}g(Z_{s,j}) - L_{t-1}(g) +  \gamma_t (\yhat_t - g(x_t'))\right\} \\
        &\leq \sup_{d_t} \ee_{\mu,\epsilon, d_t}\left\{\inf_{\yhat_t} \ee_{\gamma_t' \sim d_t}\left[ \gamma_t' \yhat_t \right] + \sup_{g \in \cF} 2\sum_{j = 1}^n \sum_{s = t+1}^T \epsilon_{s,j}g(Z_{s,j}) - L_{t-1}(g) - \gamma_t \cdot g(x_t')  \right\} \\
        &\leq \sup_{d_t} \ee_{\mu,\epsilon, d_t}\left\{ \sup_{g \in \cF} 2\sum_{j = 1}^n \sum_{s = t+1}^T \epsilon_{s,j}g(Z_{s,j}) - L_{t-1}(g) +  \ee_{\gamma_t' \sim d_t}\left[ \gamma_t' g(x_t) \right] - \gamma_t \cdot g(x_t')  \right\} \\
        &\leq \sup_{d_t} \ee_{\mu,\epsilon, d_t}\left\{ \sup_{g \in \cF} 2\sum_{j = 1}^n \sum_{s = t+1}^T \epsilon_{s,j}g(Z_{s,j}) - L_{t-1}(g) +  2 \epsilon_t \gamma_t \cdot g(x_t')  \right\} \\
        &\leq \ee_{\mu,\epsilon, d_t}\left\{ \sup_{g \in \cF} 2\sum_{j = 1}^n \sum_{s = t+1}^T \epsilon_{s,j}g(Z_{s,j}) - L_{t-1}(g) +  2 \epsilon_t \cdot g(x_t')  \right\},
    \end{align}
    where the penultimate inequality follows from symmetrization and the final inequality follows from contraction.  We now observe that because
    \begin{align}
        n \geq 8 \log(T) (f')^{-1}\left( \frac{1}{\alpha \sigma} \right),
    \end{align}
    we have
    \begin{align}
        n \geq 4 \log(n T) (f')^{-1}\left( \frac{1}{\alpha \sigma} \right)
    \end{align}
    by the fact that $2\log(n) \leq n$ for all $n > 0$.  We now apply Lemma \ref{lem:coupling} and note that by the definition of $n$, we have a coupling between $x_t$ and $Z_{t,j}$ for $1 \leq j \leq n$ such that with probability at least $1 - (nT)^{-2}$, there exists some $\jstar$ such htat $\tv\left( p_t, P_{Z_{t, \jstar}} \right) \leq \alpha$; let $\cE$ denote the event that such a $\jstar$ exists.  We now use the fact that $p_t$ is smooth and take expectations with respect to the previously arbitrary $x_t' \sim p_t$.  The above work then implies
    \begin{align}
        \ee_{x_t' \sim p_t}&\inf_{q_t \in \Delta([-1,1])} \sup_{y_t' \in[-1,1]} \left\{ \ee_{\yhat_t\sim q_t}[\ell(\yhat_t, y_t')] + \rel\left( \cF | x_1,y_1 \dots, x_t', y_t' \right) \right\}\\
        &\leq \ee_{x_t' \sim p_t}\ee_{\mu,\epsilon, d_t}\left\{ \sup_{g \in \cF} 2\sum_{j = 1}^n \sum_{s = t+1}^T \epsilon_{s,j}g(Z_{s,j}) - L_{t-1}(g) +  2 \epsilon_t \cdot g(x_t')  \right\} \\
        &=\ee_{x_t' \sim p_t}\ee_{\mu,\epsilon, d_t}\left\{\I[\cE] \sup_{g \in \cF} 2\sum_{j = 1}^n \sum_{s = t+1}^T \epsilon_{s,j}g(Z_{s,j}) - L_{t-1}(g) +  2 \epsilon_t \cdot g(x_t')  \right\} \\
        &+ \ee_{x_t' \sim p_t}\ee_{\mu,\epsilon, d_t}\left\{\I[\cE^c] \sup_{g \in \cF} 2\sum_{j = 1}^n \sum_{s = t+1}^T \epsilon_{s,j}g(Z_{s,j}) - L_{t-1}(g) +  2 \epsilon_t \cdot g(x_t')  \right\}.
    \end{align}
    Note that for the second term above, the expression in the integrand is at most $n(T - t + 1)$ and thus
    \begin{align}
        \ee_{x_t' \sim p_t}&\ee_{\mu,\epsilon, d_t}\left\{\I[\cE^c] \sup_{g \in \cF} 2\sum_{j = 1}^n \sum_{s = t+1}^T \epsilon_{s,j}g(Z_{s,j}) - L_{t-1}(g) +  2 \epsilon_t \cdot g(x_t')  \right\} \\
        &\leq \pp\left( \cE^c \right) \cdot n(T - t + 1) \leq \frac{1}{n T}
    \end{align}
    For the first term, we have
    \begin{align}
        \ee_{x_t' \sim p_t}&\ee_{\mu,\epsilon, d_t}\left\{\I[\cE] \sup_{g \in \cF} 2\sum_{j = 1}^n \sum_{s = t+1}^T \epsilon_{s,j}g(Z_{s,j}) - L_{t-1}(g) +  2 \epsilon_t \cdot g(x_t')  \right\} \\
        &= \ee_{x_t' \sim p_t}\ee_{\mu,\epsilon, d_t}\left\{\I[\cE] \sup_{g \in \cF} 2\sum_{j = 1}^n \sum_{s = t+1}^T \epsilon_{s,j}g(Z_{s,j}) - L_{t-1}(g) +  2 \epsilon_t \cdot g(Z_{t,\jstar}) + 2\epsilon_t (g(x_t') - g(Z_{t, \jstar}))  \right\} \\
        &\leq 2\tv\left( p_t, P_{Z_{t,\jstar}} \right) +  \ee_{x_t' \sim p_t}\ee_{\mu,\epsilon, d_t}\left\{ \sup_{g \in \cF} 2\sum_{j = 1}^n \sum_{s = t+1}^T \epsilon_{s,j}g(Z_{s,j}) - L_{t-1}(g) +  2 \epsilon_t \cdot g(Z_{t,\jstar})\right\}.
    \end{align}
    The first term above is at most $2 \alpha$.  For the second term, we apply Jensen's inequality and get
    \begin{align}
        \ee_{x_t' \sim p_t}&\ee_{\mu,\epsilon, d_t}\left\{ \sup_{g \in \cF} 2\sum_{j = 1}^n \sum_{s = t+1}^T \epsilon_{s,j}g(Z_{s,j}) - L_{t-1}(g) +  2 \epsilon_t \cdot g(Z_{t,\jstar})\right\} \\
        &\leq \ee_{x_t' \sim p_t}\ee_{\mu,\epsilon, d_t}\left\{ \sup_{g \in \cF} 2\sum_{j = 1}^n \sum_{s = t+1}^T \epsilon_{s,j}g(Z_{s,j}) - L_{t-1}(g) +  2 \epsilon_t \cdot g(Z_{t,\jstar}) + \sum_{j \neq \jstar} 2 \ee[\epsilon_{t,j} g(Z_{t,j})]\right\} \\
        &\leq \ee_{x_t' \sim p_t}\ee_{\mu,\epsilon, d_t}\left\{ \sup_{g \in \cF} 2\sum_{j = 1}^n \sum_{s = t}^T \epsilon_{s,j}g(Z_{s,j}) - L_{t-1}(g) \right\}.
    \end{align}
    Thus we see that
    \begin{align}
        \ee_{x_t' \sim p_t}&\ee_{\mu,\epsilon, d_t}\left\{ \sup_{g \in \cF} 2\sum_{j = 1}^n \sum_{s = t+1}^T \epsilon_{s,j}g(Z_{s,j}) - L_{t-1}(g) +  2 \epsilon_t \cdot g(x_t')  \right\} \\
        &\leq \ee_{x_t' \sim p_t}\ee_{\mu,\epsilon, d_t}\left\{ \sup_{g \in \cF} 2\sum_{j = 1}^n \sum_{s = t}^T \epsilon_{s,j}g(Z_{s,j}) - L_{t-1}(g) \right\} + 2 \alpha + \frac{n(T - t + 1)}{n^2 T^2}.
    \end{align}
    Plugging in the definition of our relaxation from the statement of the lemma concludes the proof.
\end{proof}
While Lemma \ref{lem:validrelaxation} provides an algorithm for achieving low regret, it is not clear that it is oracle efficient, due to the necessity of evaluating the expectation.  Thus, as is done in \citet{block2022smoothed}, we use the random playout idea of \citet{rakhlin2012relax} to give an oracle efficient algorithm.  Before we proceed, we recall the classical observation that, due to the convexity in the first argument of the loss function $\ell$, it suffices to suppose that $\ell$ is linear; indeed, we can simply replace the loss by the gradient of the loss at each time step and the regret of an algorithm with this latter feedback upper bounds the regret with the original loss function due to convexity.  For more details on this classical argument, see, for example, \citet[Section 5]{rakhlin2012relax} or \citet[Appendix G.2]{haghtalab2022oracle}.  Thus, we restrict our focus to linear loss and have the following result:
\begin{theorem}\label{thm:relaxations}
    Suppose that $\ell$ is a loss function convex and Lipschitz in the first argument and let $\cF: \cX \to [-1,1]$ denote a function class.  Consider an algorithm that at each time $t$, samples $Z_{s,j} \sim \mu$ for $t+1 \leq s \leq T$ and $1 \leq j \leq n$ and plays
    \begin{align}
        \yhat_t = \argmin_{\yhat \in [-1,1]} \sup_{y_t \in [-1,1]} \left\{ \ell(\yhat, y_t) + \sup_{g \in \cG}\left[ 12 \sum_{j = 1}^n \sum_{s = t+1}^T \epsilon_{s,j} g(Z_{s,j}) - \sum_{s = 1}^t \partial \ell(g(x_s), y_s) \cdot g(x_s) \right] \right\}.
    \end{align}
    Suppose that
    \begin{align}
        n \geq 8 \log(T) \cdot (f')^{-1}\left( \frac 1{\epsilon \sigma} \right)
    \end{align}
    and the adversary is $(f, \sigma)$-smoothed.  Then the learner experiences
    \begin{align}\label{eq:thmrelaxationsstatement}
        \ee\left[ \reg_T \right] \leq 2 \ee_\mu\left[ \rad_{n T}(\cF) \right] + \epsilon T + 1.
    \end{align}
    Moreover, $\yhat_t$ can be evaluated with 2 calls to $\ermoracle$ per round $t$.
\end{theorem}
\begin{proof}
    We begin by noting that it suffices to consider linear loss $\ell(\yhat, y) = \frac{1 - \yhat \cdot y}{2}$.  Indeed this is the standard reduction from online convex optimization to online linear optimization found throughout the literature.  For more details on this classical argument, see, for example, \citet[Section 5]{rakhlin2012relax} or \citet[Appendix G.2]{haghtalab2022oracle}.  Thus, we restrict our focus to linear loss and assume that $\partial \ell(g(x_s), y_s) = - y_s \cdot g(x_s) = \ell(g(x_s))$.  We now observe that two oracle calls suffice in order to evaluate $\yhat_t$, as noted in \citet{rakhlin2012relax} or \citet[Lemma 26]{block2022smoothed}.  Thus, it suffices to show that for all $(f, \sigma)$-smooth $p_t$, it holds that
    \begin{align}
        \ee_{x_t\sim p_t}&\left[ \sup_{y_t \in [-1,1]} \left\{ \ell(\yhat, y_t) + \sup_{g \in \cG}\left[ 6 \sum_{j = 1}^n \sum_{s = t+1}^T \epsilon_{s,j} g(Z_{s,j}) - \sum_{s = 1}^t \partial \ell(g(x_s), y_s) \cdot g(x_s) \right] \right\} \right] \\
        &\leq \relT{t-1}
    \end{align}
    for $\relT{t-1}$ defined as in Lemma \ref{lem:validrelaxation}.  Indeed, if this holds, then \citet[Proposition 1]{rakhlin2012relax} ensures that the final regret is bounded by $\rel(\cF | \emptyset)$, which is exactly the expression given in \eqref{eq:thmrelaxationsstatement}.  The bound in the above display, however, holds from applying the proof of \citet[Theorem 7]{block2022smoothed} and Lemma \ref{lem:validrelaxation}.  The result follows.
\end{proof}
We can now show that \Cref{thm:binaryimproper} holds as a special case:
\begin{proof}[Proof of \Cref{thm:binaryimproper}]
    Note that it is a classical fact \citep{wainwright2019high} that if $\cF$ is a binary valued class, then
    \begin{align}
        \ee_\mu\left[ \rad_T(\cF) \right] \lesssim \sqrt{\vc{\cF} \cdot T}.
    \end{align}
    The result then follows by applying \Cref{thm:relaxations}.
\end{proof}

\subsubsection{Proper Algorithm through FTPL}
We now turn to a proper algorithm, the suggested instantiation of Follow the Perturbed Leader (FTPL) from \citet{block2022smoothed}.  Due to the technical difficulties of the proof, we restrict our focus to binary valued function classes $\cF$ with linear loss $\ell(\yhat, y) = \frac{1 - \yhat \cdot y}{2}$.  Recall that we denote by $L_t(g)$ the cumulative loss of function $g \in \cF$ on the data $x_1, y_1, \dots, x_t, y_t$.  The algorithm proposed in \citet{block2022smoothed} proceeds by, at each round, sampling $\gamma_{t,1}, \dots, \gamma_{t,m}$ independent standard Gaussian random variables as well as $Z_{t,1}, \dots, Z_{t,m} \sim \mu$ and calling the oracle to evaluate
\begin{align}\label{eq:gtdef}
    g_t \in \argmin_{g \in \cF} L_{t-1}(g) + \eta \omega_{t,m}(g)
\end{align}
where
\begin{align}
    \omega_{t,m}(g) = \frac 1{\sqrt m} \sum_{i  =1}^m \gamma_{t,i} g(Z_{t,i}).
\end{align}
Note that this procedure is proper as it does not depend on $x_t$.  The player then plays $g_t(x_t)$.  We will show that this algorithm achieves no regret against a Renyi-smoothed adversary.
\begin{theorem}\label{thm:ftpl}
    Suppose that $\cF: \cX \to \{\pm 1\}$ is a binary valued function class and $\ell$ is the linear loss function.  Suppose that for some $\lambda \geq 2$, the adversary is $(f,\sigma)$ smoothed for Renyi divergence of order $\lambda$, i.e., $e^{(\lambda  -1)\dren{p_t}{\mu}} \leq \frac 1\sigma$.  If the learner plays the improper algorithm \eqref{eq:gtdef}, then
    \begin{align}
        \ee\left[ \reg_T \right] = \widetilde{O}\left(\sqrt{\vc{\cF}} \cdot T^{\frac{2 \lambda + 1}{4 \lambda - 1}} \cdot \sigma^{-\frac{1}{4\lambda - 1}} \right).
    \end{align}
\end{theorem}
Before proving the main result, we need to recall several intermediate facts.  As in the case of the improper algorithm, we will modify the technique of \citet{block2022smoothed} to our setting and apply Lemma \ref{lem:coupling}.  The first result that we need is the classic Be-the-Leader lemma from \citet{kalai2005efficient}; we will state it in the following form:
\begin{lemma}[Lemma 32 from \citet{block2022smoothed}]\label{lem:btl}
    Suppose that we are in the situation of \Cref{thm:ftpl} and let $(x_1', y_1'), \dots, (x_T', y_T')$ be tuples such that, conditional on the history, $(x_t, y_t)$ and $(x_t', y_t')$ are independent and identically distributed (in other words, the $x_t', y_t'$ form a tangent sequence).  Then the expected regret of the learner playing as in \eqref{eq:gtdef} is upper bounded by
    \begin{align}
        2 \eta \ee\left[ \sup_{g \in \cF} \omega_{1,m}(g) \right] + \sum_{t = 1}^T \ee\left[ \ell(g_t(x_t'), y_t') - \ell(g_{t+1}(x_t'), y_t') \right] + \sum_{t = 1}^T \ee\left[ \ell(g_{t+1}(x_t'), y_t') - \ell(g_{t+1}(x_t), y_t) \right].
    \end{align}
\end{lemma}
The first term controls the size of the perturbation, the second term is called the stability term in \citet{block2022smoothed}, and the last term is referred to as the generalization error.  The first term can be easily controlled:
\begin{lemma}\label{lem:perturbation}
    Suppose we are in the situation of \Cref{thm:ftpl}.  Then
    \begin{align}
        \ee\left[ \sup_{g \in \cF} \omega_{1,m}(g) \right] \lesssim \sqrt{\vc{\cF}}.
    \end{align}
\end{lemma}
\begin{proof}
    We are controlling the supremum of a Gaussian process indexed by elements in $\cF$.  An elementary chaining argument found, for example, in \citet{wainwright2019high,van2014probability} immediately yields the claim.
\end{proof}
For the second term, we need to modify an argument of \citet{block2022smoothed} in order to account for our weaker assumption.  We first use the following fact:
\begin{lemma}[Lemma 33 from \citet{block2022smoothed}]\label{lem:gaussianstability}
    Suppose we are in the situation of \Cref{thm:ftpl} and let $\muhat_m$ denote the empirical distribution of $Z_{t,1}, \dots, Z_{t,m}$.  Then for any $\alpha > 0$, it holds that
    \begin{align}
        \pp\left( \sup_{x_t, y_t} \norm{g_t - g_{t+1}}_{L^2(\muhat_m)} > \alpha\right) \lesssim \frac{1}{\alpha^4 \eta^2} + \frac{1}{\alpha^2 \eta} \ee\left[ \sup_{g \in \cF} \omega_{t,m}(g) \right].
    \end{align}
\end{lemma}
We now apply Lemma \ref{lem:gaussianstability} to prove a modified version of \citet[Lemma 34]{block2022smoothed}, allowing for $f$-smoothed adversaries.
\begin{lemma}\label{lem:expectedstability}
    Suppose that we are in the situation of \Cref{thm:ftpl}.  Suppose further that for some $\Delta > 0$, it holds that 
    \begin{align}
        \sup_{g, g' \in \cF} \abs{\norm{g- g'}_{L^2(\mu)}^2 - \norm{g - g'}_{L^2(\muhat_m)}^2} \leq \Delta.
    \end{align}
    Then
    \begin{align}
        \ee\left[ \ell(g_t(x_t), y_t) - \ell(g_{t+1}(x_t'),y_t') \right] \lesssim  \ee\left[ 1 + \sup_{g \in \cF} \omega_{t,m}(g) \right] \log(\eta) \sigma^{-\frac{1}{\lambda - 1}} \cdot \frac{1}{\eta^{1 - \frac 1\lambda}} + \sigma^{- \frac 1{\lambda - 1}} \cdot \Delta^{\frac{\lambda - 1}{\lambda}}.
    \end{align}
\end{lemma}
\begin{proof}
    We apply the technique of \citet{block2022smoothed}.  By the fact that $(x_t',y_t')$ is identically distributed as $(x_t, y_t)$, the fact that $g_t$ is independent of $(x_t,y_t)$, and the linearity of expectation, it suffices to prove the result replacing $(x_t, y_t)$ with $(x_t', y_t')$.  For any $0 < \beta < \alpha$, we compute
    \begin{align}
        \ee_{x_t' \sim p_t}&\left[ \ell(g_t(x_t'), y_t') - \ell(g_{t+1}(x_t'), y_t') \I\left[ \beta < \norm{g_t - g_{t+1}}_{L^2(\muhat_m)} \leq \alpha \right] \right]\\
        &\leq \ee_{x_t' \sim p_t} \left[ \abs{g_t(x_t') - g_{t+1}(x_t')} \I\left[ \beta < \norm{g_t - g_{t+1}}_{L^2(\muhat_m)} \leq \alpha \right]\right] \\
        &= \ee_{Z_t \sim \mu} \left[ \frac{dp_t}{d\mu}(Z_t)\abs{g_t(Z_t) - g_{t+1}(Z_t)} \I\left[ \beta < \norm{g_t - g_{t+1}}_{L^2(\muhat_m)} \leq \alpha \right]\right] \\
        &\leq e^{\dren{p_t}{\mu}} \cdot \left( \ee_{Z_t \sim \mu}\left[ \abs{g_t(Z_t) - g_{t+1}(Z_{t})}^{\frac{\lambda}{\lambda - 1}}\I\left[ \beta < \norm{g_t - g_{t+1}}_{L^2(\muhat_m)} \leq \alpha \right]  \right] \right)^{\frac{\lambda - 1}{\lambda}} \\
        &= \sigma^{- \frac{1}{\lambda - 1}} \cdot  \left( \ee_{Z_t \sim \mu}\left[ \abs{g_t(Z_t) - g_{t+1}(Z_{t})}^{2}\I\left[ \beta < \norm{g_t - g_{t+1}}_{L^2(\muhat_m)} \leq \alpha \right]  \right] \right)^{\frac{\lambda - 1}{\lambda}} \\
        &\leq \sigma^{- \frac{1}{\lambda  -1}} \cdot \pp\left( \norm{g_t - g_{t+1}}_{L^2(\muhat_m)} > \beta \right) \cdot \left( \alpha^2 + \Delta \right)^{\frac{\lambda - 1}{\lambda}},
    \end{align}
    where the first inequality follows by Lipschitzness, the second by Holder's inequality, and the last by our assumptions; the second equality follows because $g_t, g_{t+1}$ take values in $\{0,1\}$.  We now apply the summing argument from \citet{block2022smoothed} and compute, letting $S = \left\lceil \log \min\left(\sqrt{\eta}, \frac 1{\sqrt{\Delta}}\right)\right\rceil$ and $\alpha_i = 2^{1-i}$:
    \begin{align}
        \ee&\left[ \ell(g_t(x_t'), y_t') - \ell(g_{t+1}(x_t'), y_t') \right] \\
        &\leq \ee\left[ \left( \ell(g_t(x_t'), y_t') - \ell(g_{t+1}(x_t'), y_t')  \right) \I\left[ \norm{g_t - g_{t+1}}_{L^2(\muhat_m)} \leq \alpha_S \right] \right] \\
        &+ \sum_{i = 0}^{S-1} \ee\left[  \left( \ell(g_t(x_t'), y_t') - \ell(g_{t+1}(x_t'), y_t')  \right) \I\left[ \alpha_{i+1} < \norm{g_t - g_{t+1}} \leq \alpha_{i} \right] \right] \\
        &\leq \sigma^{- \frac{1}{\lambda  -1}} \left( \alpha_S^2 + \Delta \right)^{\frac{\lambda - 1}{\lambda}} + \sum_{i  =0}^{S-1} \sigma^{- \frac{1}{\lambda  -1}} \cdot \pp\left( \norm{g_t - g_{t+1}}_{L^2(\muhat_m)} > \alpha_{i+1} \right) \cdot \left( \alpha_i^2 + \Delta \right)^{\frac{\lambda - 1}{\lambda}} \\
        &\leq \sigma^{- \frac{1}{\lambda  -1}} \left( \alpha_S^2 + \Delta \right)^{\frac{\lambda - 1}{\lambda}} + C\sum_{i  =0}^{S-1} \sigma^{- \frac{1}{\lambda  -1}} \cdot \left(   \frac{1}{\alpha_{i+1}^4 \eta^2} + \frac{1}{\alpha_{i+1}^2 \eta} \ee\left[ \sup_{g \in \cF} \omega_{t,m}(g) \right]  \right) \cdot \left( \alpha_i^2 + \Delta \right)^{\frac{\lambda - 1}{\lambda}} \\
        &\leq \sigma^{- \frac{1}{\lambda  -1}} \left( \alpha_S^2 + \Delta \right)^{\frac{\lambda - 1}{\lambda}} + 2 C\sum_{i  =0}^{S-1} \sigma^{- \frac{1}{\lambda  -1}} \cdot \left(   \frac{1}{\alpha_{i+1}^{2 + \frac{2}{\lambda}} \eta^2} + \frac{1}{\alpha_{i+1}^{\frac 2\lambda} \eta} \ee\left[ \sup_{g \in \cF} \omega_{t,m}(g) \right]  \right) \\
        &\leq C' \ee\left[ 1 + \sup_{g \in \cF} \omega_{t,m}(g) \right] \log(\eta) \sigma^{-\frac{1}{\lambda - 1}} \cdot \frac{1}{\eta^{1 - \frac 1\lambda}} + C'\sigma^{- \frac 1{\lambda - 1}} \cdot \Delta^{\frac{\lambda - 1}{\lambda}}.
    \end{align}
    where we used the fact that $\alpha_i^2 \geq \Delta$ and $\frac{1}{\alpha_i \eta} \leq \sqrt{\eta}$.  The result follows.
\end{proof}
A standard empirical process theory approach allows us to bound $\Delta$; we will simply cite \citet[Lemma 36]{block2022smoothed}.  Finally, we need to bound the generalization error.  The following lemma does this:
\begin{lemma}\label{lem:generalizationerror}
    Suppose that we are in the situation of \Cref{thm:ftpl} and $\eta \geq \sqrt{m}$.  Suppose further that there is some $k \in \bbN$ satisfying
    \begin{align}
        m \geq 4k \log(T) \cdot (\epsilon \sigma)^{- \frac 1{\lambda - 1}}.
    \end{align}
    Then
    \begin{align}
        \ee\left[ \ell(g_{t+1}(x_t'), y_t') - \ell(g_{t+1}(x_t), y_t) \right] \leq 2k \epsilon + \frac{2}{k} \rad_{k}\left( \cF \right) + \frac{2}{T}
    \end{align}
\end{lemma}
\begin{proof}
    We begin by noting that by the assumption on $m$ and Lemma \ref{lem:coupling}, with probability at least $1 - (mT)^{-2}$, there exist $k$ indices $i_1, \dots, i_k$ such that $i_j \in \left\{ 1 + (j-1) \cdot \frac{m}{k}, \dots, j \cdot \frac{m}{k} \right\}$ and $\tv\left( P_{Z_{t,i_j}}, p_t \right) \leq \epsilon$ for all $1 \leq j \leq k$.  To see this, note that if
    \begin{align}
        m \geq 4k \log(T) \cdot (\epsilon \sigma)^{- \frac 1{\lambda - 1}}
    \end{align}
    then
    \begin{align}
        m \geq 2k \log(mT) \cdot (\epsilon \sigma)^{- \frac 1{\lambda - 1}}.
    \end{align}
    by the fact that $2 \log(m) \leq m$.  Let $\widetilde{\omega}_{t,m}$ denote the Gaussian process $\omega_{t,m}$ modified such that $Z_{t,i_j}$ is replaced with $Z_{t,i_j}'$, where $Z_{t, i_j'} \sim p_t$; now, let
    \begin{align}
        \gtil_{t+1} = \argmin_{g \in\cF} L_t(g) + \widetilde{\omega}_{t,m}.
    \end{align}
    Note that a union bound tells us that
    \begin{align}
        \tv\left( P_{g_{t+1}}, P_{\gtil_{t+1}} \right) \leq k \epsilon.
    \end{align}
    By \citet[Lemma 38]{block2022smoothed}, it holds that if $\eta \geq \sqrt m$, we have
    \begin{align}
        \ee\left[ \ell(\gtil_{t+1}(x_t'), y_t') - \ell(\gtil(x_t), y_t) \right] \leq \frac{2}{k} \rad_{k}\left( \cF \right) + \frac{2m + T}{(m T)^2}.
    \end{align}
    Thus, putting everything together, we have
    \begin{align}
        \ee\left[ \ell(g_{t+1}(x_t'), y_t') - \ell(g_{t+1}(x_t), y_t) \right] &= \ee\left[ \ell(g_{t+1}(x_t'), y_t') - \ell(\gtil_{t+1}(x_t'), y_t') \right] \\
        &+ \ee\left[ \ell(\gtil_{t+1}(x_t'), y_t') - \ell(\gtil(x_t), y_t) \right] + \ee\left[ \ell(\gtil_{t+1}(x_t), y_t) - \ell(g_{t+1}(x_t), y_t)  \right] \\
        &\leq 2 k \epsilon + \frac{2}{k} \rad_{k}\left( \cF \right) +\frac{2}{T}.
    \end{align}
    The result follows.
\end{proof}
We are now ready to combine our lemmata and prove the main result:
\begin{proof}[Proof of \Cref{thm:ftpl}]
    We begin by appealing to Lemma \ref{lem:btl}, which tells us that the expected regret is bounded by 
    \begin{align}
        2 \eta \ee\left[ \sup_{g \in \cF} \omega_{1,m}(g) \right] + \sum_{t = 1}^T \ee\left[ \ell(g_t(x_t'), y_t') - \ell(g_{t+1}(x_t'), y_t') \right] + \sum_{t = 1}^T \ee\left[ \ell(g_{t+1}(x_t'), y_t') - \ell(g_{t+1}(x_t), y_t) \right].
    \end{align}
    By Lemma \ref{lem:perturbation}, the first term is $O\left( \eta \sqrt{\vc{\cF}} \right)$.  For the second term, we first observe that by \citet[Lemma 36]{block2022smoothed}, we may take with probability at least $1 - T^{-1}$,
    \begin{align}
        \Delta \lesssim \frac{1}{\sqrt{m}}\left( \frac{1}{m} \rad_m(\cF) + \sqrt{\frac{\log\left( \frac 1\delta \right)}{m}} \right) \lesssim \frac{\sqrt{\vc{\cF} + \log\left( T \right)}}{m}
    \end{align}
    in Lemma \ref{lem:expectedstability}, where the latter inequality follows from Proposition \ref{prop:rudelsonvershynin}.  Thus, applying Lemma \ref{lem:generalizationerror}, we see that the expected regret satisfies
    \begin{align}
        \ee\left[ \reg_T \right] &\lesssim \eta \sqrt{\vc{\cF}} + \sqrt{\vc{\cF}} \cdot \log(\eta) \sigma^{-\frac{1}{\lambda - 1}} \cdot \frac{T}{\eta^{1 - \frac 1\lambda}} + T \cdot \sigma^{- \frac 1{\lambda - 1}} \cdot \left(\frac{\sqrt{\vc{\cF} + \log\left( T \right)}}{m}  \right)^{\frac{\lambda - 1}{\lambda}} \\
        &\quad + T k \epsilon + \frac{T}{k} \rad_{k}\left( \cF \right) + 1 \\
        &\lesssim \eta \sqrt{\vc{\cF}} + \sqrt{\vc{\cF}} \cdot \log(\eta) \sigma^{-\frac{1}{\lambda - 1}} \cdot \frac{T}{\eta^{1 - \frac 1\lambda}} \\
        &\quad + T \cdot \sigma^{- \frac 1{\lambda - 1}} \cdot \left(\frac{\sqrt{\vc{\cF} + \log\left( \frac{1}{\delta} \right)}}{k \log(T) \cdot (\epsilon \sigma)^{- \frac 1{\lambda - 1}}}  \right)^{\frac{\lambda - 1}{\lambda}} + Tk \epsilon + T \cdot \sqrt{\frac{ \vc{\cF}}{k}},
    \end{align}
    where the last step again came from Proposition \ref{prop:rudelsonvershynin}.  Setting
    \begin{align}
        k = \epsilon^{- \frac 23} && \eta = \sqrt{m} && m = k \log(T) \cdot (\epsilon \sigma)^{- \frac 1{\lambda - 1}}  && \epsilon = T^{- \frac{6\lambda - 6}{4 \lambda - 1}} \cdot \sigma^{-\frac{3}{4 \lambda - 1}}
    \end{align}
    and plugging in concludes the proof.
\end{proof}

\end{document}